\documentclass[journal]{IEEEtran}
\usepackage{amsmath}
\usepackage{graphicx}
\usepackage{url}
\usepackage{hyperref}
\usepackage{enumerate,enumitem}
\usepackage{wrapfig}
\usepackage{algorithm}
\usepackage{algorithmic}
\usepackage{booktabs}
\usepackage{multicol,multirow}
\usepackage{subfigure}
\usepackage{amssymb}
\usepackage{color}
\usepackage[table]{xcolor}

\graphicspath{{figs/}}

\def\lyu{\textcolor{black}}

\def\mb{\mathbf}
\def\mbb{\mathbb}
\def\mc{\mathcal}

\def\etal{{\em et al.\/}\, }

\def\ie{\textit{i.e.}}

\newtheorem{lemma}{Lemma}
\newtheorem{proof}{Proof}

\begin{document}
%
\title{Variational Continual Test-Time Adaptation}
%
%
%

\author{Fan~Lyu,~\IEEEmembership{Member,~IEEE,}
        Kaile~Du,~\IEEEmembership{Student Member,~IEEE,}
        Yuyang~Li,
        Hanyu~Zhao,
        Fuyuan~Hu,~\IEEEmembership{Member,~IEEE,}
        Zhang~Zhang$^\dagger$,~\IEEEmembership{Member,~IEEE,}
        Guangcan~Liu,~\IEEEmembership{Senior Member,~IEEE,}
        and~Liang~Wang,~\IEEEmembership{Fellow,~IEEE}
\thanks{F. Lyu is with the Computer Vision Center (CVC), Universitat Autònoma de Barcelona, Barcelona 08193, Spain. Email: fanlyu@cvc.uab.cat}
\thanks{Z. Zhang and L. Wang are with the New Laboratory of Pattern Recognition, State Key Laboratory of Multimodal Artificial Intelligence Systems, Institute of Automation, Chinese Academy of Sciences, Beijing, 100019 China. E-mail: \{zzhang, wangliang\}@nlpr.ia.ac.cn.}
\thanks{K. Du, Y, Li and G. Liu are with the School of Automation, Southeast University, Nanjing, 210096 China. E-mail: \{kailedu, yuyangli\}@seu.edu.cn, gcliu1982@gmail.com.}
\thanks{H. Zhao is with China Nuclear Power Engineering Co.,LTD, Beijing, 100840 China. E-mail: zhaohyb@cnpe.cc}
\thanks{F. Hu is with the School of Electronic \& Engineering, Suzhou University of Science and Technology, Suzhou, 215000 China. E-mail: fuyuanhu@mail.usts.edu.cn.}
\thanks{$^\dagger$ Correspond author.}
\thanks{Manuscript received April 19, 2005; revised August 26, 2015.}}

%
%

\markboth{Journal of \LaTeX\ Class Files,~Vol.~14, No.~8, August~2015}%
{Shell \MakeLowercase{\textit{et al.}}: Bare Demo of IEEEtran.cls for IEEE Journals}
%



\maketitle

\begin{abstract}
Continual Test-Time Adaptation (CTTA) task investigates effective domain adaptation under the scenario of continuous domain shifts during testing time. 
Due to the utilization of solely unlabeled samples, there exists significant uncertainty in model updates, leading CTTA to encounter severe error accumulation issues.
In this paper, we introduce VCoTTA, a variational Bayesian approach to measure uncertainties in CTTA. 
At the source stage, we transform a pretrained deterministic model into a Bayesian Neural Network (BNN) via a variational warm-up strategy, injecting uncertainties into the model. 
During the testing time, we employ a mean-teacher update strategy using variational inference for the student model and exponential moving average for the teacher model. 
Our novel approach updates the student model by combining priors from both the source and teacher models. 
The evidence lower bound is formulated as the cross-entropy between the student and teacher models, along with the Kullback-Leibler (KL) divergence of the prior mixture. 
Experimental results on three datasets demonstrate the method's effectiveness in mitigating error accumulation within the CTTA framework.
\end{abstract}


%
\IEEEpeerreviewmaketitle

\section{Introduction}
%
%
%
%
\IEEEPARstart{W}{hen} deploying a pre-trained model for real-world applications, environmental changes are common. Recently, test-time adaptation (TTA)~\cite{liu2023test,verma2025real} tasks have garnered scholarly attention and made significant progress, focusing on how pre-trained models can automatically adapt to domain changes during testing.
However, traditional TTA is typically suited for single-domain changes, while the Continual Test Time Adaptation (CTTA) task extends this concept to long-term scenarios.
CTTA~\cite{wang2022continual} aims to enable a model to accommodate a sequence of distinct distribution shifts during the testing time, making it applicable to various risk-sensitive applications in open environments, such as autonomous driving and medical imaging. However, real-world non-stationary test data exhibit high uncertainty in their temporal dynamics~\cite{huang2022extrapolative}, \lyu{presenting a major challenge named error accumulation}~\cite{wang2022continual}.

\begin{figure}[t]
\centering
\includegraphics[width=\linewidth]{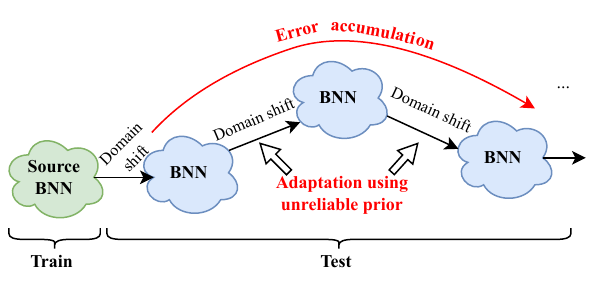}
\vspace{-15px}
\label{fig:unreliable}
\caption{
In CTTA task, a BNN model is first trained on a source dataset, and then is used to adapt to updated with unreliable priors, which may suffer from noisy update.
In this paper, we use the Bayesian approach to measure the uncertainty and try to reduce the effect of unreliable priors, achieving better adaptation.
}
\vspace{-10px}
\end{figure}

\lyu{To solve the error accumulation issue}, existing CTTA studies rely on methods that enforce prediction confidence, such as entropy minimization. 
The majority of the existing methods focus on improving the confidence of the source model during the testing phase.
For example, some methods employ the mean-teacher architecture~\cite{tarvainen2017mean} to mitigate error accumulation, where the student learns to align with the teacher and the teacher updates via moving average with the student. 
As to the challenge of forgetting source knowledge, some methods adopt augmentation-averaged predictions~\cite{wang2022continual, brahma2023probabilistic,dobler2023robust,yang2023exploring} for the teacher model, strengthening the teacher's confidence to reduce the influence from highly out-of-distribution samples.
Some other methods, such as~\cite{dobler2023robust,chakrabarty2023sata}, propose to adopt the contrastive loss to maintain the already learned semantic information.
Moreover, some researchers believe that the source model is more reliable, thus they are designed to restore the source parameters~\cite{wang2022continual,brahma2023probabilistic}.
Though the above methods keep the model from confusion of vague pseudo labels, they may suffer from overly confident predictions that are less calibrated.
\lyu{\textit{As a result, these methods cannot adaptively distinguish between reliable and unreliable samples, which leaves them prone to error accumulation and limits their ability to quantify risks during prediction.}}

In CTTA, the reliable estimation of uncertainty becomes particularly crucial in the context of continual distribution shift~\cite{ovadia2019can}. It is meaningful to design a model capable of encoding the uncertainty associated with temporal dynamics and effectively handling distribution shifts.
To address this problem, we refer to the Bayesian Inference (BI)~\cite{box2011bayesian}, which retains a distribution over model parameters that indicates the plausibility of different settings given the observed data, and it has been witnessed as effective in traditional continual learning (CL) tasks~\cite{nguyen2018variational}, which also perform in long-term processes.
In Bayesian continual learning, the posterior in the last learning task is set to be the current prior, which will be multiplied by the current likelihood.
This kind of prior transmission is designed to reduce catastrophic forgetting in continual learning~\cite{wickramasinghe2023continual}.
However, this is not feasible in CTTA because unlabeled data may introduce \textbf{unreliable priors}.
As shown in Fig.~\ref{fig:unreliable}, an unreliable prior may lead to a poor posterior, which may then propagate errors to the next inference, leading to the issue of error accumulation.
The objective of this paper is to devise a CTTA procedure that not only enhances predictive accuracy under distribution shifts but also provides reliable uncertainty estimates by using the Bayesian method.



Thus, we delve into the utilization of BI framework to evaluate model uncertainty in CTTA, aiming to mitigate the impact of unreliable priors and reduce the error propagation.
To approximate the intractable likelihood in BI, we adopt to use online Variational Inference (VI)~\cite{wang2011online,sato2001online}, and accordingly name our method Variational Continual Test-Time Adaptation (VCoTTA).
At the source stage, we first transform a pretrained deterministic model into a Bayesian Neural Network (BNN) by a variational warm-up strategy, where the local reparameterization trick~\cite{kingma2015variational} is used to inject uncertainties into the source model.
During the testing phase, we employ a mean-teacher update strategy, where the student model is updated via VI and the teacher model is updated by the exponential moving average. 
Specifically, for the update of the student model, we propose to use a mixture of priors from both the source and teacher models, then the Evidence Lower BOund (ELBO) becomes the cross-entropy between the student and teachers plus the KL divergence of the prior mixture.
\lyu{We demonstrate the effectiveness of the proposed method on three benchmark datasets, where the results show that it can mitigate the error accumulation caused by unreliable priors in CTTA, leading to clear performance improvements.}

Our contributions are three-fold:

\begin{enumerate}[label=(\arabic*),left=0pt,itemsep=0pt]
    \item \lyu{We propose {VCoTTA}, a general framework that formulates CTTA as online variational inference within a BNN. This formulation explicitly models uncertainty and mitigates the error accumulation that arises from unreliable priors under persistent distribution shifts.}
    \item \lyu{We design a {variational warm-up strategy} that converts off-the-shelf pre-trained models into BNNs, enabling uncertainty injection while preserving prior knowledge. This allows effective adaptation without retraining from scratch.}
    \item \lyu{We introduce an {adaptive entropy-based prior mixture}, which dynamically fuses the source prior and teacher prior to provide a more reliable updating mechanism. Unlike ensemble-based methods, our mixture weights are determined by uncertainty, ensuring robustness in continuously changing environments.}
\end{enumerate}

\section{Related Work}
\label{sec:rel}

\subsection{Continual Test-Time Adaptation}

TTA enables the model to dynamically adjust to the characteristics of the test data, i.e. target domain, in a source-free and online manner \cite{jain2011online,sun2020test,wang2020tent}. 
Recently, CTTA~\cite{wang2022continual} has been introduced to tackle TTA within a continuously changing target domain, involving long-term adaptation~\cite{kim2021domain}. This configuration often grapples with the challenge of error accumulation~\cite{tarvainen2017mean,wang2022continual}. Specifically, prolonged exposure to unsupervised loss from unlabeled test data during long-term adaptation may result in significant error accumulation. Additionally, as the model is intent on learning new knowledge, it is prone to forgetting source knowledge, which poses challenges when accurately classifying test samples similar to the source distribution.


Some methods employ the mean-teacher architecture~\cite{tarvainen2017mean} to mitigate error accumulation, where the student learns to align with the teacher and the teacher updates via moving average with the student. 
CoTTA~\cite{wang2022continual} used weight-averaged and augmentation-averaged predictions along with stochastic restoration of some neurons to the source pre-trained weights.
ECoTTA~\cite{song2023ecotta} proposed a memory-efficient approach, which utilizes lightweight meta networks and a self-distilled regularization technique to minimize memory consumption.
Some methods propose to adopt the contrastive loss to maintain the already learnt semantic information.
RMT~\cite{dobler2023robust} utilized symmetric cross-entropy loss for mean teachers and contrastive learning to address error accumulation.
SWA~\cite{yang2023exploring} introduced a framework that enhances pseudo-label learning through self-adaptive thresholding, soft-weighted contrastive learning, and soft weight alignment to address the challenges of noisy pseudo-labels and knowledge forgetting.
Some methods believe that the source model is more reliable, thus they are designed to use the source parameters more effectively.
Gan \etal\cite{gan2023decorate} introduced a visual domain prompt approach, which learns lightweight image-level visual prompts to reformulate input data for adapting to changing target domains without fine-tuning the source model.
PETAL~\cite{brahma2023probabilistic} proposed a probabilistic framework, which employs a student-teacher model, a regularizer based on the source model's posterior, and a data-driven parameter restoration technique using Fisher information matrix.
VIDA~\cite{liu2023vida} employed high-rank and low-rank embedding spaces to manage domain-specific and domain-shared knowledge, respectively, and utilizes a homeostatic knowledge allotment strategy to dynamically fuse the knowledge from these adapters for improved adaptation performance.
ADMA~\cite{liu2024continual} employed distribution-aware masking and histograms of oriented gradients reconstruction to enhance target domain knowledge extraction.
\lyu{Though the above methods alleviate the confusion caused by vague pseudo labels, they often lead to over-confident predictions with poor calibration. This indicates that existing deterministic strategies lack mechanisms to reliably handle sample-wise uncertainty, leaving room for approaches that explicitly incorporate uncertainty estimation.}

\subsection{Bayesian Neural Network}

Bayesian framework is natural to incorporate past knowledge and sequentially update the belief with new data~\cite{zhao2022deep}.
The bulk of work on Bayesian deep learning has focused on scalable approximate inference methods. 
These methods include stochastic VI~\cite{hernandez2015probabilistic,louizos2017multiplicative}, dropout~\cite{gal2016dropout,kingma2015variational} and Laplace approximation~\cite{ritter2018scalable,friston2007variational} etc., and leveraging the stochastic gradient descent (SGD) trajectory, either for a deterministic approximation or sampling.
In a BNN, we specify a prior $p(\boldsymbol{\theta})$ over the neural network parameters, and compute the posterior distribution over parameters conditioned on training data, $p(\boldsymbol{\theta}|\mc{D}) \propto p(\boldsymbol{\theta})p(\mc{D}|\boldsymbol{\theta})$. This procedure should give considerable advantages for reasoning about predictive uncertainty, which is especially relevant in the small-data setting.

Crucially, when performing BI, we need to choose a prior distribution that accurately reflects the prior beliefs about the model parameters before seeing any data~\cite{gelman1995bayesian,fortuin2021bayesian}.
In conventional static machine learning, the most common choice for the prior distribution over the BNN weights is the simplest one: the isotropic Gaussian distribution. 
However, this choice has been proved indeed suboptimal for BNNs~\cite{fortuin2021bayesian}.
Recently, some studies estimate uncertainty in continual learning within a BNN framework, such as \cite{nguyen2018variational,ebrahimi2019uncertainty,farquhar2019unifying,kurle2019continual}. They set the current prior to the previous posterior to mitigate catastrophic forgetting. 
However, the prior transmission is not reliable in the unsupervised CTTA task.
Any prior mistakes will be enlarged by adaptation progress, manifesting error accumulation.
To solve the {unreliable prior} problem, this paper proposes a prior mixture method based on VI.

\section{Variational Inference in CTTA}
\label{sec:method}

\subsection{BI in traditional CL and in CTTA}
We start from the supervised BI in typical continual learning, where the model aims to learn multiple classification tasks in sequence.
Let $\mc{D}=\{(x_n,y_n)\}_{n=1}^N$ be the training set, where $x_n$ and $y_n$ denotes the training sample and the corresponding class label. The task $t$ is to learn a direct posterior approximation over the model parameter $\boldsymbol{\theta}$ as follows.
\begin{equation}
p(\boldsymbol{\theta}|\mathcal{D}_{1:t})~\propto~ p_t(\boldsymbol{\theta})p(\mathcal{D}_t|\boldsymbol{\theta}),
\end{equation}
where $p(\boldsymbol{\theta}|\mathcal{D}_{1:t})$ denotes the posterior of sequential tasks on the learned parameter and $p(\mathcal{D}_t|\boldsymbol{\theta})$ is the likelihood of the current task.
The current prior $p_t(\boldsymbol{\theta})$ is regarded as the given knowledge.
\cite{nguyen2018variational} proposes that this current prior can be the posterior learned in the last task, \ie,  $p_t(\boldsymbol{\theta})=p(\boldsymbol{\theta}|\mathcal{D}_{1:t-1})$, where the inference becomes 
\begin{equation}
p(\boldsymbol{\theta}|\mathcal{D}_{1:t})~\propto~ p(\boldsymbol{\theta}|\mathcal{D}_{1:t-1})p(\mathcal{D}_t|\boldsymbol{\theta}).
\end{equation}

In contrast to continual learning~\cite{dai2025prompt}, CTTA faces a sequence of learning tasks in test time without any label information, requiring the model to adapt to each novel domain sequentially.
In this case, we assume that each domain is i.i.d. and the classes are separable following many unsupervised studies~\cite{miller1996mixture, van2020survey,caron2018deep}.
We use $\mc{U}=\{x_n\}_{n=1}^N$ to represent the unlabeled test dataset.
The CTTA model is first trained on a source dataset $\mc{D}_0$, and then adapted to unlabeled test domains starting from $\mc{U}_1$.
For the $t$-th adaptation, we have
\begin{equation}
p(\boldsymbol{\theta}|\mathcal{U}_{1:t}\cup\mathcal{D}_0)\propto p_t(\boldsymbol{\theta})p(\mathcal{U}_t|\boldsymbol{\theta}).
\end{equation}
Similarly, we can set the last posterior to be the current prior, \ie, $p_t(\boldsymbol{\theta})= p(\boldsymbol{\theta}|\mathcal{U}_{1:t-1}\cup \mathcal{D}_0)$ {and $p_1(\boldsymbol{\theta})= p(\boldsymbol{\theta}|\mathcal{D}_0)$}.
{However, employing BI for adaptation on unlabeled testing data can result in untrustworthy posterior estimates. Therefore, during subsequent adaptation, the untrustworthy posterior automatically transforms into unreliable priors, leading to error accumulation.}
In other words, an unreliable prior $p_t(\boldsymbol{\theta})$ will make the current posterior even less trustworthy.
Moreover, the joint likelihood $p(\mathcal{U}_t|\boldsymbol{\theta})$ for $t>0$ is intractable on unlabeled data.
The BI in CL and in CTTA can be compared in Fig.~\ref{fig:vcl_vs_vctta}.
For CL, BI is conducted by the posterior propagation, that is, the prior of the next task is equal to the current posterior.
This is feasible in supervised CL, where the data label is provided.
For CTTA, the posterior is not trustworthy using only pseudo labels to adapt to a new domain.
Thus, propagating the untrustworthy posterior to the next stage would make it unreliable prior, which will result in error accumulation.


\begin{figure}[t]
\centering
\subfigure[BI in continual learning]{\centering
  \includegraphics[width=\linewidth]{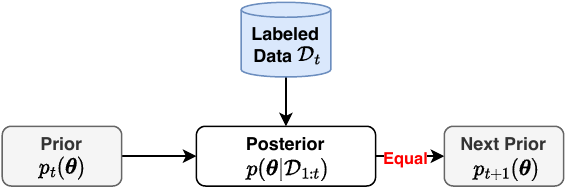}
  \label{fig:vcl}
  }
\subfigure[BI in CTTA (Our method)]{\centering
\includegraphics[width=\linewidth]{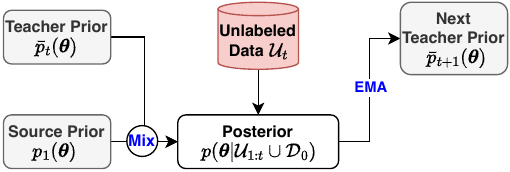}
\label{fig:vctta}}
\vspace{-10px}
\caption{BI in continual learning versus CTTA. We find the traditional prior transmission is infeasible in CTTA because of the unreliable prior from unlabeled data. In our method, we place CTTA in a mean-teacher structure, and design BI in CTTA using a mixture of teacher prior and source prior. The next teacher prior is updated by the exponential moving average.}
\label{fig:vcl_vs_vctta}
\vspace{-10px}
\end{figure}

\subsection{VI in CTTA}

To make the BI feasible in CTTA task, in this paper, we transform the question to an easy-to-compute form.
Referring to \cite{grandvalet2004semi}, the unsupervised inference can be transformed into 
\begin{equation}
    p(\boldsymbol{\theta}|\mc{U})\propto p(\boldsymbol{\theta})\exp\left(-{\lambda}{H}(\mc{U}|\boldsymbol{\theta})\right),
    \label{eq:bi4ctta}
\end{equation}
where $H$ denotes the conditional entropy and $\lambda$ is a scalar hyperparameter to weigh the entropy term.
{This simple form reveals that the prior belief about the conditional entropy of labels is given by the inputs.}
The observation of the input $\mc{U}$ provides information on the drift of the input distribution, which can be used to update the belief over the learned parameters $\boldsymbol{\theta}$ through Eq.~\eqref{eq:bi4ctta}. Consequently, this allows the utilization of unlabeled data for BI.
More detailed derivations can be seen in Section~\ref{app:bi}.









In a BNN, the posterior distribution is often intractable and some approximation methods are required, even when calculating the initial posterior. 
In this paper, we leverage online VI, as it typically outperforms the other methods for complex models in the static setting~\cite{bui2016deep}.
VI defines a variational distribution $q(\boldsymbol{\theta})$ to approxmiate the posterior $p(\boldsymbol{\theta}|\mc{U})$.
The approximation process is as follows.
\begin{equation}
q_t(\boldsymbol{\theta})=\arg\min_{q\in\mbb{Q}}\mathrm{KL}\left[q(\boldsymbol{\theta})\parallel\frac1{Z_t}p_t(\boldsymbol{\theta})e^{-{\lambda}{H}(\mc{U}_t|\boldsymbol{\theta})}\right],
\label{eq:varapprox}
\end{equation}
where $\mbb{Q}$ is the distribution searching space and $Z_t$ is the intractable normalizing hyperparameter.
Specifically, the KL divergence is expanded as
\begin{equation}
\begin{aligned}
&\mathrm{KL}\left[q(\boldsymbol{\theta})\parallel\cfrac1{Z_t}p_{t}(\boldsymbol{\theta})e^{-\lambda H(\mc{U}_t|\boldsymbol{\theta})}\right] \\
= & -\int_{\boldsymbol{\theta}} q(\boldsymbol{\theta})\log\cfrac{\cfrac{1}{Z_t}p_{t}(\boldsymbol{\theta})e^{-\lambda H(\mc{U}_t|\boldsymbol{\theta})}}{q(\boldsymbol{\theta})} d \boldsymbol{\theta} \\
= & -\int_{\boldsymbol{\theta}} q(\boldsymbol{\theta})\log\cfrac{1}{Z_t}e^{-\lambda H(\mc{U}_t|\boldsymbol{\theta})} d \boldsymbol{\theta} - \int_{\boldsymbol{\theta}} q(\boldsymbol{\theta})\log\cfrac{p_{t}(\boldsymbol{\theta})}{q(\boldsymbol{\theta})} d \boldsymbol{\theta} \\
= & \int_{\boldsymbol{\theta}} q(\boldsymbol{\theta})\log Z_t d \boldsymbol{\theta} + \lambda\int_{\boldsymbol{\theta}} q(\boldsymbol{\theta})H(\mc{U}_t|\boldsymbol{\theta}) d \boldsymbol{\theta} \\ & - \int_{\boldsymbol{\theta}} q(\boldsymbol{\theta})\log\cfrac{p_{t}(\boldsymbol{\theta})}{q(\boldsymbol{\theta})} d \boldsymbol{\theta} \\
= & \log Z_t + \lambda\mathbb{E}_{\boldsymbol{\theta}\sim q(\boldsymbol{\theta})}H(\mc{U}_t|\boldsymbol{\theta}) + \mathrm{KL}\left(q(\boldsymbol{\theta})\parallel p_{t}(\boldsymbol{\theta})\right),
\end{aligned}
\end{equation}
where the first constant term can be reduced in the optimization.
Thus, the ELBO can be computed by 
\begin{equation}
    {\text{ELBO}} = -\underbrace{\lambda\mathbb{E}_{ \boldsymbol{\theta}\sim q(\boldsymbol{\theta})}H(\mathcal{U}_t| \boldsymbol{\theta})}_{\lyu{\text{Entropy term}}} - \underbrace{\text{KL}\left( q(\boldsymbol{\theta})||p_{t}( \boldsymbol{\theta})\right)}_{\lyu{\text{KL term}}}.
    \label{eq:elbo}
\end{equation}
Optimizing with Eq.~\eqref{eq:elbo} makes the model adapt to domain shift.
While VI offers a good framework for measuring uncertainty in CTTA, it is noteworthy that VI does not directly address the issue of unreliable priors. The error accumulation remains a significant concern.

Despite this, the form of the ELBO in variational inference offers a pathway for mitigating the impact of unreliable priors.
In Eq.~\eqref{eq:elbo}, the \textit{entropy term} may result in overly confident predictions that are less calibrated, while the \textit{KL term} may be directly affected by an unreliable prior.
In the following section, we will discuss how to solve the problems when computing the two terms.


\section{Adaptation and Inference in VCoTTA}

\label{sec:vcotta}

\subsection{\lyu{Entropy term: Integrating Mean Teacher into VI}}
\label{sec:mt}
\lyu{In the above section, we show that unreliable priors are a key challenge in VI-based CTTA.
To address this, our solution is to design an \emph{uncertainty-aware prior mixture}, where uncertainty determines how much the model should rely on the source prior or the teacher prior. 
To implement this mechanism in practice, we adopt the standard Mean Teacher (MT) framework~\cite{tarvainen2017mean} as a backbone, following prior works in CTTA~\cite{wang2022continual}. 
It is important to note that MT itself is not our contribution, but a commonly used mechanism that provides the teacher distribution for our Bayesian formulation.}
MT is initially proposed in semi-supervised and unsupervised learning, where the teacher model guides the unlabeled data, helping the model generalize and improve performance with the utilization of large-scale unlabeled data.
Because of these advantages, MT is used in some CTTA methods such as \cite{wang2022continual}, thus we also build our method on it.

MT structure is composed of a student model and a teacher model, where the student model learns from the teacher and the teacher updates using Exponential Moving Average (EMA)~\cite{hunter1986exponentially}.
In VI, the student is set to be the variational distribution $q(\boldsymbol{\theta})$, which is a Gaussian mean-field approximation for its simplicity. It is achieved by stacking the biases and weights of the network as follows.
\begin{equation}
q(\boldsymbol{\theta})=\prod\nolimits_d\mathcal{N}\left(\boldsymbol{\theta}_{d};\mu_{d},\text{diag}(\sigma_{d}^2)\right),
\end{equation}
where $d$ denotes each dimension of the parameter.
\lyu{The teacher model $\bar{p}(\boldsymbol{\theta})$ is also modeled as a Gaussian distribution and serves as one component of the prior mixture. 
Importantly, $\bar{p}(\boldsymbol{\theta})$ itself is not guaranteed to be reliable. 
Instead, our Bayesian formulation adaptively weights the contribution of the teacher and the source prior according to uncertainty.
In practice, this adaptive weighting is realized by guiding the student distribution to align with the teacher through a cross-entropy (CE) loss:}
\begin{equation}
    L_{\text{CE}}(q, \bar{p}) = - \mbb{E}_{\boldsymbol{\theta}\sim q(\boldsymbol{\theta})}\mbb{E}_{x\sim \mc{U}}\left[\bar{p}(x|\boldsymbol{\theta}) \log q(x|\boldsymbol{\theta})\right].
\end{equation}
\lyu{In practice, following~\cite{wang2019symmetric}, we also experiment with the Symmetric Cross-Entropy (SCE) loss to improve training stability:}
\begin{equation}
\begin{aligned}
    L_{\text{SCE}}(q, \bar{p}) = - \mbb{E}_{\boldsymbol{\theta}\sim q(\boldsymbol{\theta})}\mbb{E}_{x\sim \mc{U}}\big[\bar{p}(x|\boldsymbol{\theta}) \log q(x|\boldsymbol{\theta}) \\ +q(x|\boldsymbol{\theta}) \log \bar{p}(x|\boldsymbol{\theta})\big].
\end{aligned}
\label{eq:sce}
\end{equation}
\lyu{SCE is not our contribution, but an auxiliary choice that balances gradients for high- and low-confidence samples.}

\subsection{KL term: Mixture-of-Gaussian Prior}

\label{sec:mixprior}

For the KL term, to reduce the impact of {unreliable prior}, we propose a mixing-up approach to combining the teacher and source prior adaptatively.
The source prior is warmed up upon the pretrained deterministic model $p_1(\boldsymbol{\theta}) = p(\boldsymbol{\theta}|\mathcal{D}_0)$ (see Section~\ref{sec:warmup}).
The teacher model $\bar{p}_t(\boldsymbol{\theta})$ is updated by EMA (see Section~\ref{sec:ema}).
We assume that the prior should be the mixture of the two Gaussian priors.
Using only the source prior, the adaptation is limited.
While using only the teacher prior, the prior is prone to be unreliable.

We use the mean entropy derived from a given series of data augmentation to represent the confidence of the two prior models and mix up the two priors with a modulating factor
\begin{equation}
    \alpha = \frac{1}{|\mc{I}|}\sum\nolimits_{i\in\mc{I}}\frac{e^{H(x| \boldsymbol{\theta}_0)/\tau}}{e^{H(x| \boldsymbol{\theta}_0)/\tau} + e^{H(x| \bar{\boldsymbol{\theta}})/\tau}},
    \label{eq:alpha}
\end{equation}
where $\mc{I}$ denotes augmentation types. 
{$\boldsymbol{\theta}_0$ and $\bar{\boldsymbol{\theta}}$ are the parameters of the source model and the teacher model. $\tau$ is the temperature factor.} 
{Thus, as shown in Fig.~\ref{fig:vctta}, the current prior  $p_t(\boldsymbol{\theta})$} is set to the mixture of priors as 
\begin{equation}
    p_t(\boldsymbol{\theta}) = 
    \alpha\cdot p_1(\boldsymbol{\theta}) + (1-\alpha) \cdot\bar{p}_t(\boldsymbol{\theta}).
\end{equation}
In the VI, we use the upper bound to update the KL term \cite{liu2019variational} (see Section~\ref{app:mix}) for simplicity, 
\begin{equation}
    \text{KL}\left( q||p_{t}\right) \le \alpha\cdot\text{KL}\left( q||{p}_{0}\right) + (1-\alpha)\cdot\text{KL}\left( q||\bar{p}_{t}\right).
\end{equation}
Furthermore, we also improve the teacher-student alignment in the entropy term (see Eq.~\eqref{eq:sce}) by picking up the augmented logits with a larger confidence than the raw data.
That is, we replace the teacher log-likelihood $\log\bar{p}(x|\boldsymbol{\theta})$ by 
\begin{equation}
    \log{\bar{p}'(x|\boldsymbol{\theta})} =
    \frac{
        \sum_{i\in\mc{I}}\mb{1}\left(f(\bar{p}(x'_i))>f(\bar{p}(x))+\epsilon\right)\cdot\log\bar{p}(x'_i)}
        {\sum_{i\in\mc{I}}\mb{1}\left(f(\bar{p}(x'_i))>f(\bar{p}(x))+\epsilon\right)},
    \label{eq:teacher_aug}
\end{equation}
where, for brevity, we let $\bar{p}(x'_i)=\bar{p}(x'_i|\boldsymbol{\theta})$ and $\bar{p}(x)=\bar{p}(x)|\boldsymbol{\theta})$ in short.
$f(\cdot)$ is the confidence function.
$\epsilon$ denotes the confidence margin and $\mb{1}(\cdot)$ is an indicator function.
Eq.~\eqref{eq:teacher_aug} can be regarded as a filter, meaning that for each sample, the reliable teacher is represented by the average of its augmentations with $\epsilon$ more confidence.
In Section~\ref{app:whymix}, we prove that the proposed mixture-of-Gaussian is benifical to CTTA.
In the experiments of Section~\ref{app:margin}, we discuss the influence of different $\epsilon$.

\subsection{Adaptation and Inference}

\subsubsection{Variational Warm-up}
\label{sec:warmup}

To obtain a source BNN, instead of training a model from scratch on the source data $\mc{D}_0$, we transform a pretrained deterministic model to a BNN by variational warm-up strategy.
Specifically, we leverage the local reparameterization trick~\cite{kingma2015variational} to add stochastic parameters, and warm up the model: 
\begin{equation}
q_0(\boldsymbol{\theta})=\arg\min_{q\in\mbb{Q}}\mathrm{KL}\left[q(\boldsymbol{\theta})\parallel\frac1{Z_0}p(\boldsymbol{\theta})p(\mc{D}_0|\boldsymbol{\theta})\right],
\label{eq:vwu}
\end{equation}
{where $p(\boldsymbol{\theta})$ represents the prior distribution, say the pretrained deterministic model.
}
Eq.~\eqref{eq:vwu} denotes a standard VI on the source data, and we optimize the ELBO to obtain the variational distribution~\cite{wang2011online}.
By the variational warm-up, we can easily transform an off-the-shelf pretrained model into a BNN with a stochastic dynamic. The variational warm-up strategy is outlined in Algorithm \ref{alg:vwu}.

\begin{figure}[t]
    \begin{algorithm}[H]
   \caption{Variational warm-up}
   \label{alg:vwu}
\begin{algorithmic}[1]
   \STATE {\bfseries Input:} Source data $\mc{D}_0$, pretrained model $p_0(\boldsymbol{\theta})$
   \STATE Initialize prior distribution $p(\boldsymbol{\theta})$ with $p_0(\boldsymbol{\theta})$
   \STATE Update $p(\boldsymbol{\theta}|\mc{D}_0) \approx q_0(\boldsymbol{\theta})$ by $p(\boldsymbol{\theta})$ and $\mc{D}_0$ using Eq.~\eqref{eq:vwu}
   \STATE {\bfseries Output:} Source prior $p_1(\boldsymbol{\theta}) = p(\boldsymbol{\theta}|\mc{D}_0)$
\end{algorithmic}
\end{algorithm}
\vspace{-20px}
\end{figure}

The warm-up strategy is a common approach in TTA and CTTA tasks to further build knowledge structure for the source model, such as \cite{jung2023cafa,song2023ecotta,dobler2023robust,chen2023each}.
Some other methods may not use warm-up but still use the source data, such as \cite{niu2022efficient}.
The warm-up strategy uses the source data only before deploying the model to CTTA scenario, and it is regarded as a part of pretraining. 
All these methods using source data are operationalized source-free at test time and found beneficial to CTTA.
We use the warm-up to inject the uncertainties into a given source model, i.e., turning an off-the-shelf pretrained model into a pretrained BNN model. 
This is convenient to obtain a pretrained BNN, because the warm-up strategy uses only a few epochs.
We offer more discussions and experiments on the proposed variational warm-up strategy in Section~\ref{app:vwu}.

\subsubsection{Student update via VI}
\label{sec:stu}

The student model $q_t(\boldsymbol{\theta})$ is adapted by approximating using Eq.~\eqref{eq:varapprox}, and is optimized on:
\begin{equation}
        L(q_t) =L_{\text{CE}}(q_t, \bar{p}'_t) 
         + \alpha\cdot\text{KL}\left( q_t||q_{0}\right) + (1-\alpha)\cdot\text{KL}\left( q_t||\bar{q}_{t}\right),
    \label{eq:update_stu}
\end{equation}
where $\bar{p}'_t$ is the current augmented teacher model in Eq.~\eqref{eq:teacher_aug}, 
{and $p_1(\boldsymbol{\theta})\approx q_0(\boldsymbol{\theta})$, $\bar{p}_t(\boldsymbol{\theta})\approx \bar{q}_{t}(\boldsymbol{\theta})$.}
The KL term between two Gaussians can be computed in a closed form.

\subsubsection{Teacher update via EMA}
\label{sec:ema}
The teacher model is updated using EMA.
Let $(\boldsymbol{\mu}, \boldsymbol{\sigma})$ and $(\bar{\boldsymbol{\mu}}, \bar{\boldsymbol{\sigma}})$ be the mean and standard deviation of the student and teacher model, respectively.
At test time, {the teacher model $\bar{q}_{t}(\boldsymbol{\theta})$} is updated by
\begin{align}
    \bar{\boldsymbol{\mu}} \leftarrow \beta \bar{\boldsymbol{\mu}} + (1-\beta)\boldsymbol{\mu}, \quad
    \bar{\boldsymbol{\sigma}} \leftarrow \beta \bar{\boldsymbol{\sigma}} + (1-\beta)\boldsymbol{\sigma}.
    \label{eq:update_tea}
\end{align}
Although the std is not used in the cross entropy to compute the likelihood, the teacher prior distribution is important to adjust the student distribution via the KL term.

\subsubsection{Model inference}

\label{sec:method.infe}

At any time, CTTA model needs to predict and adapt to the unlabeled test data.
In our VCoTTA, we also use the mixed prior to serve as the inference model.
That is, for a test data point $x$, the model inference is represented by
\begin{equation}
\begin{aligned}
p_t(x) = &\int p(x|\boldsymbol{\theta})p_t(\boldsymbol{\theta}) d\boldsymbol{\theta} 
\\= & \int \alpha p(x|\boldsymbol{\theta})p_1(\boldsymbol{\theta}) +(1-\alpha) p(x|\boldsymbol{\theta})\bar{p}_t(\boldsymbol{\theta})   d\boldsymbol{\theta},
\end{aligned}
\label{eq:inference}
\end{equation}
For the data prediction, the model only uses the expectation to reduce the stochastic, but leverages stochastic dynamics in domain adaptation.

\subsection{The algorithm}

\begin{figure}[t]
\begin{algorithm}[H]
   \caption{Variational CTTA}
   \label{alg:vcotta}
\begin{algorithmic}[1]
   \STATE {\bfseries Input:} Source data $\mc{D}_0$, pretrained model $p_0(\boldsymbol{\theta})$, Unlabeled test data from different domain $\mc{U}_{1:T}$
   \STATE $p_1(\boldsymbol{\theta})=\text{Variational warm-up}(\mc{D}_0, p_0(\boldsymbol{\theta}))$. // Alg.~\ref{alg:vwu} 
   \FOR{Domain shift $t=1$ {\bfseries to} $T$}
   \FOR{Test data $x\sim \mc{U}_t$}
   \STATE Model predict for $x$ (Eq.~\eqref{eq:inference})
   \STATE Update student model using $x$ (Eq.~\eqref{eq:update_stu})
   \STATE Update teacher model via EMA (Eq.~\eqref{eq:update_tea})
   \ENDFOR
   \ENDFOR
\end{algorithmic}
\end{algorithm}
\vspace{-20px}
\end{figure}

We illustrate the whole algorithm in Algorithm \ref{alg:vcotta}.
We first transform an off-the-shelf pretrained model into BNN via the variational warm-up strategy (Section~\ref{sec:warmup}).
After that, we obtain a BNN, and for each domain shift, we forward and adapt each test data point in an MT architecture.
For a data point $x$, we first predict the class label using the mixture of the source model and the teacher model (Section~\ref{sec:method.infe}).
Then, we update the student model using VI, where we use cross entropy to compute the entropy term and use the mixture of priors for the KL term (Section~\ref{sec:stu}).
Finally, we update the BNN teacher model via EMA (Section~\ref{sec:ema}).
The process is feasible for any test data without labels.

Then, we show the whole CTTA algorithmic workflow utilizing variational approximation in VCoTTA:

\textbf{Before testing}: First, we adopt a variational warm-up strategy to inject stochastic dynamics into the model before adaptation.
Given the source dataset $\mc{D}_0$, we can use a variational approximation of $p(\boldsymbol{\theta}|\mathcal{D}_0)$ as Eq.~\eqref{eq:vwu},
where we use the pretrained deterministic model $p_0(\boldsymbol{\theta})$ as the prior distribution.

\textbf{During testing}:
We have the domain shift during the testing time, and we need to update the BNN in an MT structure.
Then, at the beginning of the test time,  we set the prior in task $t$ as $ p_t(\boldsymbol{\theta}) = \alpha\cdot p_1(\boldsymbol{\theta}) + (1-\alpha) \cdot\bar{p}_t(\boldsymbol{\theta})$ for the variational approximation, where $p_1(\boldsymbol{\theta})\approx q_0(\boldsymbol{\theta})$ and $\bar{p}_t(\boldsymbol{\theta}) \approx \bar{q}_{t}(\boldsymbol{\theta})$.
$\bar{q}_t(\boldsymbol{\theta})$ means the real-time posterior of the teacher model for the $t$-th test domain, is constantly updated by $q_t(\boldsymbol{\theta})$ via EMA (see Eq.~\eqref{eq:update_tea}). 
Note that we do not have $\bar{q}_t(\boldsymbol{\theta})$ for the first update in the $t$-th task, and use $q_{t-1}(\boldsymbol{\theta})$ construct the prior. 
This process is not required to inform the model that the domain produces a shift.
Then, with the approximation of the prior $p_t(\boldsymbol{\theta})$ , we get $q_t(\boldsymbol{\theta})$ for student model at the test domain $t$ as Eq.~\eqref{eq:varapprox}.
As a result, we recursively derive the following prior and posterior during the test time, thereby achieving the goal of VCoTTA.

\subsection{Comparisons with Existing Methods}

VCL~\cite{nguyen2018variational} is a classic continual learning (CL) study that uses VI.
VCL inspires our work but has the following differences.
First, VCL studies the supervised CL task, while our VCoTTA studies the unsupervised CTTA task. 
Second, CL only suffers from catastrophic forgetting (CF), while CTTA suffers from both CF and error accumulation.
Third, to conduct BI, one needs to compute prior and likelihood. For the prior, the current prior of VCL is set to be the previous posterior, while in CTTA such a prior may be unreliable. For the likelihood, VCL can directly compute likelihood, CTTA is under the unsupervised setting, thus in our work, we deduce the BI in CTTA using conditional entropy.
Last, to reduce error accumulation in the unsupervised scenario, we employ a mean-teacher update strategy using VI for the student model and EMA for the teacher model and compute a prior mixture to guide the student update.
Moreover, VCL maintains an extra coreset from the training set, while VCoTTA never stores any data during the test time.

PETAL~\cite{brahma2023probabilistic} also estimates uncertainties in CTTA via BI.
The BI formulation is similar between PETAL and ours, which is derived from~\cite{grandvalet2004semi}, but PETAL uses a different method to conduct the inference.
First, PETAL only does not estimate the model uncertainties via BNN.
Second, PETAL ignores the unreliable prior in CTTA, and follows the VCL setting that still uses the previous posterior as the current prior.
Third, we conduct BI using VI while PETAL use SWAG~\cite{maddox2019simple}. SWAG has advantages in terms of computational efficiency during training, but might not handle unreliable priors as effectively as VI since it doesn't explicitly model the posterior distribution.

\section{Discussion and Theory Analysis}

\subsection{BI during continual testing time}
\label{app:bi}
In this subsection, we derive the BI representation of CTTA in Eq.~\eqref{eq:bi4ctta}.
First of all, the goal of CTTA is to learn a posterior distribution $p(\boldsymbol{\theta}|\mathcal{U}_{1:T}\cup\mathcal{D}_0)$ from a source dataset $\mc{D}_0$, and a sequence of unlabeled test data from $\mc{U}_1$ to $\mc{U}_T$.
Assuming the source dataset $\mc{D}_0$ is drawn from a distribution $\phi$, and $\tilde{\phi}_t$ is the $t$-th shift of the unlabeled test data. 
Let the parameters of the model be $\boldsymbol{\theta}$, 
then following the semi-supervised learning framework~\cite{grandvalet2004semi}, we incorporate all input-generating distributions into the belief over the model parameters $\boldsymbol{\theta}$ as follows
\begin{equation}
    \begin{aligned}
        p(\boldsymbol{\theta}| \phi , \tilde{\phi}_1, \cdots, \tilde{\phi}_T) 
        \propto ~ & p(\boldsymbol{\theta})\exp\left(-\lambda_0H_{\boldsymbol\theta, \phi}(Y|{X})\right)\\ ~ &\prod_{t=1}^{T}\exp\left(-\lambda_tH_{\boldsymbol\theta, \tilde{\phi}_t}(Y|{X})\right),
    \end{aligned}
    \label{eq:ded19}
\end{equation}
{where the inputs $X$ are sampled i.i.d. from a generative model with parameters $\phi$, while the corresponding labels $Y$ are sampled from a conditional distribution $p(Y |X, \boldsymbol{\theta})$, which is parameterized by the model parameters $\boldsymbol{\theta}$}.
$p(\boldsymbol{\theta})$ is the prior distribution over $\boldsymbol{\theta}$, and
$\{\lambda_0,\lambda_1,\cdots,\lambda_T\}$ are the factors for approximation weighting.

In Eq.~\eqref{eq:ded19}, the term $H_{\boldsymbol\theta, \phi}(Y|\boldsymbol{X})$ represents the cross entropy of the supervised learning, and the term $H_{\boldsymbol\theta, \tilde{\phi}_t}(Y|\boldsymbol{X})$ for $t>0$ denotes the Shannon entropy of the unsupervised learning.
Following \cite{zhou2021training}, we can empirically use a point estimation to get a plug-in Bayesian approach to approximate the above formula:
\begin{equation}
\begin{aligned}
    &p(\boldsymbol{\theta}|\mathcal{U}_{1:T}\cup\mathcal{D}_0) \\
    \propto \quad & p(\boldsymbol{\theta})\prod_{\forall x, y \in \mc{D}_0}p(y|x,\boldsymbol{\theta})\exp\left(-\frac{\lambda_0}{|\mc{D}_0|}\sum_{\forall x \in \mc{D}_0}{H}(y|{x}, \boldsymbol{\theta})\right) \\ 
    &\prod_{t=1}^{T} \exp\left(-\frac{\lambda_t}{|\mc{U}_t|}\sum_{\forall x \in \mc{U}_t}{H}(y|{x}, \boldsymbol{\theta})\right).
\end{aligned}
\end{equation}
To make the formula feasible to CTTA, that is, no source data is available at the test time, we set $\lambda_0 = 0$. 
Since the source knowledge can be represented by $p(\boldsymbol{\theta}|\mc{D}_0) \propto p(\boldsymbol{\theta})\prod_{\forall x, y \in \mc{D}_0}p(y|x,\boldsymbol{\theta})$, for the $t$-th test domain, the Bayesian inference in CTTA can be represented as follows:
\begin{equation}
    \begin{aligned}
        &p(\boldsymbol{\theta}|\mathcal{U}_{1:t}\cup\mathcal{D}_0) \\\propto ~ & p(\boldsymbol{\theta}|\mc{D}_0)\prod_{i=1}^{t} \exp\left(-\frac{\lambda_i}{|\mc{U}_i|}\sum_{\forall x \in \mc{U}_i}{H}(y|{x}, \boldsymbol{\theta})\right) \\
        \propto ~ & p(\boldsymbol{\theta}|\mathcal{U}_{1:t-1}\cup\mathcal{D}_0)\exp\left(-\frac{\lambda_t}{|\mc{U}_t|}\sum_{\forall x \in \mc{U}_t}{H}(y|{x}, \boldsymbol{\theta})\right),
    \end{aligned}
\end{equation}
where $H(\mc{U}_t|\boldsymbol{\theta}) = \cfrac{1}{|\mc{U}_t|}\sum_{\forall x \in \mc{U}_t}{H}(y|{x}, \boldsymbol{\theta})$ and the above formula can be rewritten in simplicity as
\begin{equation}
\begin{aligned}
    p(\boldsymbol{\theta}|\mathcal{U}_{1:t}\cup\mathcal{D}_0) &\propto p(\boldsymbol{\theta}|\mathcal{U}_{1:t-1}\cup\mathcal{D}_0) e^{-\lambda H(\mc{U}_t|\boldsymbol{\theta})}\\ &= p_{t}(\boldsymbol{\theta})e^{-\lambda H(\mc{U}_t|\boldsymbol{\theta})},
\end{aligned}
\end{equation}
which specifies the Bayesian inference process on continuously arriving unlabeled data in CTTA.

\subsection{Upper Bound of the Mixture of Two KL Divergencies}

\label{app:mix}

We refer to the lemma that was stated for the mixture of Gaussian in \cite{singer1998batch}.
\begin{lemma}
The KL divergence between mixture distributions $p = \sum_{i=1}^{k}\alpha_i p_i$ and $p' = \sum_{i=1}^{k}\alpha_i p'_i$ is upper-bounded by
    \begin{equation}
        \mathrm{KL}(p\parallel p') \le \mathrm{KL}(\boldsymbol{\alpha}\parallel\boldsymbol{\alpha}') + \sum_{i=1}^{k}\alpha_i\mathrm{KL}(p_i\parallel p'_i),
    \end{equation}
    where $\boldsymbol{\alpha} = (\alpha_1, \alpha_2, \cdots, \alpha_k)$ and $\boldsymbol{\alpha}' = (\alpha'_1, \alpha'_2, \cdots, \alpha'_k)$ are the weights of the mixture components. The equality holds if and only if ${\alpha_i p_i}/{\sum_{j=1}^{k}\alpha_j p_j} = {\alpha'_i p'_i}/{\sum_{j=1}^{k}\alpha'_j p'_j}$ for all $i$.
\end{lemma}
\begin{proof}
By the log-sum inequality~\cite{cover1999elements}, we have
    \begin{equation*}
        \begin{aligned}
            & \mathrm{KL}(\sum_{i=1}^{k}\alpha_i p_i\parallel\sum_{i=1}^{k}\alpha_i p'_i) 
            =  \int\left(\sum_{i=1}^{k}\alpha_i p_i\right) \log\cfrac{\sum_{i=1}^{k}\alpha_i p_i}{\sum_{i=1}^{k}\alpha_i p'_i} \\
            \leq ~ & \int\sum_{i=1}^{k}\alpha_i p_i \log\cfrac{\alpha_i p_i}{\alpha_i p'_i} 
            =  \sum_{i=1}^{k}\alpha_i \left(\int p_i \log\cfrac{\alpha_i}{\alpha_i'} + \int p_i \log\cfrac{p_i}{p'_i}\right) \\
            = ~ & \mathrm{KL}(\boldsymbol{\alpha}\parallel\boldsymbol{\alpha}') + \sum_{i=1}^{k}\alpha_i\mathrm{KL}(p_i\parallel p'_i).
        \end{aligned}
    \end{equation*}
\end{proof}
In our algorithm, $q(\boldsymbol{\theta})$ is set to be a mixture of Gaussian distributions, \ie, $p_t(\boldsymbol{\theta}) = \alpha\cdot p_1(\boldsymbol{\theta}) + (1-\alpha) \cdot\bar{p}_t(\boldsymbol{\theta})$.
In the above inequality, let $q(\boldsymbol{\theta}) = \sum_{i=1}^{k}\alpha_i q(\boldsymbol{\theta})$, we can get the upper bound of the KL divergence between $q(\boldsymbol{\theta})$ and $p_{t}(\boldsymbol{\theta})$:
\begin{equation}
    \mathrm{KL}(q\parallel p_{t}) \leq \alpha\cdot\text{KL}\left( q||{p}_{1}\right) + (1-\alpha)\cdot\text{KL}\left( q||\bar{p}_{t}\right).
\end{equation}
So the lower bound in Eq.~\eqref{eq:elbo} can be redefined as
\begin{equation}
    \begin{aligned}
        \text{ELBO} = &-\lambda\mathbb{E}_{\boldsymbol{\theta}\sim q(\boldsymbol{\theta})}H(\mc{U}_t|\boldsymbol{\theta}) - \mathrm{KL}\left(q(\boldsymbol{\theta})\parallel p_{t}(\boldsymbol{\theta})\right) \\
        \geq &-\lambda\mathbb{E}_{\boldsymbol{\theta}\sim q(\boldsymbol{\theta})}H(\mc{U}_t|\boldsymbol{\theta}) - \alpha\cdot\text{KL}\left( q||{p}_{1}\right) \\ &- (1-\alpha)\cdot\text{KL}\left( q||\bar{p}_{t}\right).
    \end{aligned}
\end{equation}
Then, we have obtained a lower bound that can be optimized as the source prior distribution $q_0(\boldsymbol{\theta})$ and the teacher prior distribution $\bar{q}_{t}(\boldsymbol{\theta})$ are multivariate Gaussian distributions.
Specifically, when the priors are multivariate Gaussian distributions, the KL term can be computed in closed form as
\begin{equation}
\begin{aligned}
    &\mathrm{KL}\left(\mc{N}(\boldsymbol{\mu}_1, \boldsymbol{\Sigma}_1)\parallel\mc{N}(\boldsymbol{\mu}_2, \boldsymbol{\Sigma}_2)\right) \\
    = &\frac{1}{2} \Bigg( \text{tr}(\boldsymbol{\Sigma}_2^{-1} \boldsymbol{\Sigma}_1) + (\boldsymbol{\mu}_2 - \boldsymbol{\mu}_1)^\top \boldsymbol{\Sigma}_2^{-1} (\boldsymbol{\mu}_2 - \boldsymbol{\mu}_1) \\&- k + \ln \left( \frac{\det(\boldsymbol{\Sigma}_2)}{\det(\boldsymbol{\Sigma}_1)} \right) \Bigg),
\end{aligned}
    \label{eq:gausskl}
\end{equation}
where $\boldsymbol{\Sigma} = \text{diag}(\boldsymbol{\sigma^2})$, $k$ represents the dimensionality of the distributions, $\text{tr}(\cdot)$ denotes the trace of a matrix, and $\det(\cdot)$ stands for the determinant of a matrix.

\subsection{Advantage of the Mixture of Gaussian Prior}

\label{app:whymix}

In this subsection, we illustrate why the mixture of Gaussian prior are beneficial to CTTA. 
We first define what is a better distribution for CTTA.
Assume there exists an ideal prior distribution $\hat{p}_t$, which effectively represents the distribution of the model after learning all past knowledge, including that from the source and unlabeled datasets. Then we can use the difference between a distribution and the ideal distribution $\hat{p}_t$ (here we use KL divergence) to measure the goodness of a distribution, i.e., $\text{KL}(\cdot||\hat{p}_t)$.

Generally, neither the source prior $p_1$ (trained on labeled data) nor the adapted prior $\bar{p}_t$ (adapt on unlabeled data, being unreliable) can be completely consistent with $\hat{p}_t$. Considering that, as $t$ increases, the difference between $\bar{p}_t$ and $\hat{p}_t$ will increase without an upper bound due to the error accumulation (since $t$ is infinitely growing). The source prior $p_1$ cannot adapt to the unlabeled data, but it contains important information from the labeled data, and the ideal distribution cannot forget the source information too much, so we assume that the difference between $p_1$ and $\hat{p}_t$ is a constant, i.e., $\text{KL}(p_1||\hat{p}_t)<U$, where $U$ is a constant upper bound.
Accordingly, it can be considered that mixing the source prior $p_1$ and the adapted prior $\bar{p}_t$ in some way is beneficial for reducing $\text{KL}(\cdot||\hat{p}_t)$.

In our method, we use Gaussian mixture $p_t = \alpha_t p_1 + (1-\alpha_t) \bar{p}_t$, where $\alpha$ is computed by Eq.~\eqref{eq:alpha}. It is easy to illustrate the benefits of this idea using the following inequality:
\begin{equation}
\begin{aligned} 
\text{KL}(p_t||\hat{p}_t) 
&= \text{KL}\left[(\alpha_t p_1 + (1-\alpha_t) \bar{p}_t)||\hat{p}_t\right]\\ 
&\leq \alpha_t \text{KL}(p_1||\hat{p}_t) + (1-\alpha_t) \text{KL}(\bar{p}_t||\hat{p}_t)\\ 
&\leq \alpha_t U + (1-\alpha_t) \text{KL}(\bar{p}_t||\hat{p}_t). 
\end{aligned} 
\label{eq:kl_ineq}
\end{equation}
If $\text{KL}(\bar{p}_t||\hat{p}_t) \geq U$, which can be satisfied as mentioned above, then we have $ \text{KL}(p_t||\hat{p}_t) \leq \text{KL}(\bar{p}_t||\hat{p}_t)$.
This indicates that the mixed distribution $p_t$ is closer to the ideal distribution $\hat{p}_t$ than the adapted prior $\bar{p}_t$.
A similar idea can be found in the stochastic restoration in CoTTA~\cite{wang2022continual}, which randomly restores parts of the parameters of the current model into the parameters of the source model.

\section{Experiment}
\label{sec:ex}

\subsection{Experimental Setting}

\label{sec:ex.set}

\subsubsection{Dataset}
We use CIFAR10-to-CIFAR10C, CIFAR100-to-CIFAR100C, and ImageNet-to-ImageNetC as benchmarks to assess the robustness of classification models. Each dataset comprises 15 distinct types of corruption, each applied at severity 5. These corruptions are systematically applied to test images from the original CIFAR10 and CIFAR100 datasets, as well as validation images from the original ImageNet dataset.
For simplicity in tables, we use C1 to C15 to represent the 15 types of corruption, \ie,
C1: Gaussian, C2: Shot, C3: Impulse C4: Defocus, C5: Glass, C6: Motion, C7: Zoom, C8: Snow, C9: Frost, C10: Fog, C11: Brightness, C12: Contrast, C13: Elastic, C14: Pixelate, C15: Jpeg.

\begin{table*}[t]
\caption{Classification error rates (\%) for the standard CIFAR10-to-CIFAR10C and CIFAR100-to-CIFAR100C CTTA tasks. 
All models are evaluated under the highest corruption severity level (level 5) in an online adaptation setting. 
C1 to C15 denote the 15 corruption types described in Section~\ref{sec:ex.set}. 
The best result is highlighted in \textbf{bold}, and the second-best result is underlined.
}
\label{tab:cifar10c}
\vspace{-15px}
\begin{center}
\resizebox{\linewidth}{!}{
\begin{tabular}{clrcccccccccccccccc}
\toprule
&\textbf{Method} & \textbf{Venue} & \textbf{C1} & \textbf{C2} & \textbf{C3} & \textbf{C4}  & \textbf{C5} & \textbf{C6} & \textbf{C7} & \textbf{C8} & \textbf{C9} & \textbf{C10} & \textbf{C11} & \textbf{C12} & \textbf{C13} & \textbf{C14} & \textbf{C15} & \textbf{Avg}\\
\midrule
\multirow{17}{*}{\rotatebox{90}{CIFAR10-to-CIFAR10C}}&Source   &{-}& 72.3 & 65.7 & 72.9 & 46.9 & 54.3 & 34.8 & 42.0 & 25.1 & 41.3 & 26.0 & 9.3 & 46.7  &26.6 & 58.5 & 30.3 & 43.5 \\
&TENT~\cite{wang2020tent}&ICLR~2021 & 24.8 & 20.6 & 28.6 &	14.4 & 31.1 & 16.5 & 14.1 & 19.1 & 18.6 & 18.6 & 12.2 & 20.3 & 25.7 & 20.8 & 24.9 & 20.7\\
&Ada~\cite{AdaContrast}&CVPR~2022  & 29.1 & 22.5 & 30.0 & 14.0 & 32.7 & 14.1 & 12.0 & 16.6 & 14.9 & 14.4 & 8.1 & 10.0 & 21.9 & 17.7 & 20.0 & 18.5\\
&CoTTA~\cite{wang2022continual}&CVPR~2022      & 24.3 & 21.3 & 26.6 & 11.6 & 27.6 & 12.2 & 10.3 & 14.8 & 14.1 & 12.4 & 7.5 & 10.6 & 18.3 & 13.4 & 17.3 & 16.2 \\
&EATA~\cite{EATA}&ICML~2022&27.3&20.3&28.6&13.5&27.8&14.6&13.5&17.9&19.0&16.4&10.7&14.6&20.8&16.4&19.1&18.7 \\
&RoTTA~\cite{yuan2023robust}    & CVPR~2023 &30.3&25.4&34.6&18.3&34.0&14.7&11.0&16.4&14.6&14.0&8.0&12.4&20.3&16.8&19.4&19.3 \\
&PETAL~\cite{brahma2023probabilistic}&CVPR~2023&26.0&18.9&27.2&12.1&26.4&13.2&12.1&16.5&18.5&15.0&10.2&13.2&19.4&15.0&17.6&17.3 \\
&RMT \cite{dobler2023robust}&CVPR~2023&24.0 & 20.4 & 25.6 & 12.6 & 25.4 & 14.2 & 12.2 &15.4 &15.1 &14.1 &10.3 &13.7 &17.1 &13.5 &16.0 &16.7\\
&BeCOTTA~\cite{becotta}&ICLR~2024&22.9&19.1&26.9&\underline{10.2}&27.5&12.7&10.4&14.7&14.3&12.4&\underline{7.2}&\textbf{9.4}&20.9&15.2&20.2&16.3 \\
&SANTA~\cite{chakrabarty2023sata}    & TMLR~2023&23.9&20.1&28.0&11.6&27.4&12.6&10.2&14.1&13.2&12.2&7.4&10.3&19.1&13.3&18.5&16.1 \\
&SWA~\cite{yang2023exploring}    & IJCAI~2023&23.9&20.5&24.5&11.2&26.3&\underline{11.8}&\underline{10.1}&14.0&\underline{12.7}&\underline{11.5}&7.6&9.5&17.6&\underline{12.0}&15.8&15.3 \\
&DSS~\cite{wang2024continual}&WACV~2024&24.1&21.3&25.4&11.7&26.9&12.2&10.5&14.5&14.1&12.5&7.8&10.8&18.0&13.1&17.3&16.0 \\
&OBAO~\cite{obao}&ECCV~2024&23.6&19.9&26.0&11.8&25.3&13.2&10.9&14.3&13.5&12.7&9.0&11.9&17.1&12.7&15.9&15.8\\
&PALM~\cite{palm}&AAAI~2025&25.9&\underline{18.1}&\underline{22.7}&12.4&25.3&13.2&10.8&13.5&13.2&12.2&8.5&11.9&17.9&12.0&\underline{15.5}&15.5 \\
&SURGEON~\cite{surgeon}&CVPR~2025&25.3&20.0&27.4&12.0&28.0&13.7&11.1&15.3&14.3&14.1&8.2&11.8&19.2&13.5&18.7&16.8\\
&TCA~\cite{CTTA-TCA}&CVPR~2025&\underline{22.6}&19.4&23.2&11.0&\underline{23.4}&12.0&\underline{10.1}&\underline{13.3}&13.5&11.9&7.9&11.0&\underline{16.7}&12.3&15.7&\underline{14.9} \\
\cmidrule{2-19}
& VCoTTA    & Ours&\textbf{18.1} &\textbf{14.9} &\textbf{22.0} &\textbf{9.7} &\textbf{22.6} &\textbf{11.0}&\textbf{9.5}&\textbf{11.4}&\textbf{10.6}&\textbf{10.5}&\textbf{6.5}&\textbf{9.4}&{\textbf{15.6}}&\textbf{11.0}&\textbf{14.5}&\textbf{13.1} \\
\midrule
\multirow{17}{*}{\rotatebox{90}{CIFAR100-to-CIFAR100C}}
&Source  &- & 73.0&	68.0&	39.4&	29.3&	54.1&	30.8&	28.8&	39.5&	45.8&	50.3&	29.5&	55.1&	37.2	&74.7&	41.2&	46.4\\

&TENT~\cite{wang2020tent}&ICLR~2021  & 37.2 & 35.8 & 41.7 & 37.9 & 51.2 & 48.3 & 48.5 & 58.4 & 63.7 & 71.1 & 70.4 & 82.3 & 88.0 & 88.5 & 90.4 & 60.9 \\

&Ada~\cite{AdaContrast}&CVPR~2022 & 42.3 & 36.8 & 38.6 & 27.7 & 40.1 & 29.1 & 27.5 & 32.9 & 30.7 & 38.2 & 25.9 & 28.3 & 33.9 & 33.3 & 36.2 & 33.4 \\

&CoTTA \cite{wang2022continual}&CVPR~2022& 40.1 & 37.7 & 39.7 & 26.9 & 38.0 & 27.9 & 26.4 & 32.8 & 31.8 & 40.3 & 24.7 & 26.9 & 32.5 & 28.3 & 33.5 & 32.5 \\
&EATA \cite{EATA}&ICML~2022&46.7&34.5&42.4&25.9&36.8&31.6&25.3&35.6&33.8&41.1&\underline{22.6}&28.6&36.6&28.4&33.5&33.6 \\
&RoTTA~\cite{yuan2023robust}    & CVPR~2023&49.1&44.9&45.5&30.2&42.7&29.5&26.1&32.2&30.7&37.5&24.7&26.9&32.5&28.3&33.5&32.5 \\
&PETAL~\cite{brahma2023probabilistic}&CVPR~2023&42.1&38.6&38.7&28.1&38.1&28.1&27.5&31.2&30.7&37.2&26.4&28.7&29.4&28.6&31.6&32.8\\
&RMT~\cite{dobler2023robust}&CVPR~2023&40.5 &36.1 &36.3 &27.7 &\underline{33.9} &28.5 &26.4 &\underline{29.0} &29.0 &32.5 &25.1 &27.4 &28.2 &26.3 &29.3 &30.4\\
&BeCoTTA~\cite{becotta}&ICLR~2024&42.1 &38.0 &42.2 &30.2&42.9&31.7&29.8&35.1&33.9&38.5&27.9&32.0&36.7&31.6&39.9 & 35.5\\
&SANTA~\cite{chakrabarty2023sata}    &TMLR~2023&\underline{36.5}&33.1&\underline{35.1}&25.9&34.9&27.7&25.4&29.5&29.9&33.1&23.6&26.7&31.9&27.5&35.2&30.3 \\
&SWA~\cite{yang2023exploring}    & IJCAI~2023&39.4&36.4&37.4&\underline{25.0}&36.0&\underline{26.6}&25.0&29.1&28.4&35.0&23.5&\underline{25.1}&\underline{28.5}&25.8&\textbf{29.6}&30.0 \\
&DSS \cite{wang2024continual}&WACV~2024  &39.7&36.0&37.2&26.3&35.6&27.5&25.1&31.4&30.0&37.8&24.2&26.0&30.0&26.3&31.3&30.9 \\
&OBAO~\cite{obao}&ECCV~2024&38.8&35.0&35.4&26.7&\textbf{33.2}&27.4&25.4&29.4&28.8&32.8&24.1&26.1&30.9&26.9&32.0&30.2\\
&PALM~\cite{palm}&AAAI~2025&{37.3}&\textbf{32.5}&\textbf{34.9}&26.2&35.3&27.5&\underline{24.6}&29.1&29.2&34.1&23.5&27.0&31.1&26.6&34.1&30.2 \\
&SURGRON~\cite{surgeon}&CVPR~2025&40.3&35.6&38.1&26.4&36.7&28.1&25.5&31.2&30.0&35.6&24.8&27.2&31.3&28.2&36.1&31.7\\
&TCA \cite{CTTA-TCA}&CVPR~2025&38.5&36.0&36.6&25.8&34.8&27.2&25.1&30.5&\underline{27.8}&\underline{30.5}&24.1&25.7&\textbf{27.3}&26.6&30.5 &\underline{29.8}\\
\cmidrule{2-19}
& VCoTTA & Ours &\textbf{35.3}&\underline{32.8}&38.9&\textbf{23.8}&{34.6}&\textbf{25.5}&\textbf{23.2}&\textbf{27.5}&\textbf{26.7}&\textbf{30.4}&\textbf{22.1}&\textbf{23.0}&\underline{28.1}&\textbf{24.2}&\underline{30.4}&\textbf{28.4} \\
\bottomrule
\end{tabular}
}
\end{center}
\vspace{-10px}
\end{table*}

\subsubsection{Pretrained Model}
Following previous studies~\cite{wang2020tent,wang2022continual}, we adopt pretrained WideResNet-28~\cite{zagoruyko2016wide} model for CIFAR10to-CIFAR10C, pretrained ResNeXt-29~\cite{xie2017aggregated} for CIFAR100-to-CIFAR100C, and standard pretrained ResNet-50~\cite{he2016deep} for ImageNet-to-ImagenetC. 
Note in our VCoTTA~\cite{wang2022continual}, we further warm up the pretrained model to obtain the stochastic dynamics for each dataset.
Similar to CoTTA, we update all the trainable parameters in all experiments. 
The augmentation number is set to 32 for all compared methods that use the augmentation strategy.


\subsection{Comparison Results}

\label{sec:ex.major}

\lyu{We compare our VCoTTA with multiple state-of-the-art (SOTA) methods, including 
TENT~\cite{wang2020tent},
Ada~\cite{AdaContrast}
CoTTA~\cite{CTTA-CoTTA}, 
EATA~\cite{EATA},
RoTTA~\cite{yuan2023robust}
PETAL~\cite{brahma2023probabilistic},
RMT~\cite{dobler2023robust}, 
BeCOTTA~\cite{becotta},
SANTA~\cite{chakrabarty2023sata},
SWA~\cite{yang2023exploring},
DSS~\cite{wang2024continual},
OBAO~\cite{obao},
PALM~\cite{palm},
SURGEON~\cite{surgeon},
and TCA~\cite{CTTA-TCA}.
\textsc{Source} denotes the baseline pre-trained model without any adaptation. 
All compared methods adopt the same backbone, pre-trained model and hyperparameters.}


\lyu{We present major comparisons with state-of-the-art methods on {CIFAR10C}, {CIFAR100C} and ImageNetC, and summarize the main findings. 
First, source-only models without adaptation suffer heavily from domain shifts, which clearly demonstrates the necessity of CTTA. 
Second, traditional TTA methods that ignore continual shifts, such as TENT, perform poorly in this setting. 
Third, mean-teacher based approaches including CoTTA and PETAL confirm the usefulness of pseudo-labeling, although they remain vulnerable to error accumulation, which gradually produces unreliable pseudo labels and negative transfer as adaptation progresses. 
As shown in Table~\ref{tab:cifar10c}, on {CIFAR10C} and {CIFAR100C}, VCoTTA achieves consistent and clear improvements over baselines. 
For example, VCoTTA reduces the error rate to 13.1\% compared with 15.3\% from SWA on CIFAR10C, and to 28.4\% compared with 30.0\% from SWA on CIFAR100C. 
These gains highlight the effectiveness of explicitly modeling uncertainty and leveraging the prior mixture mechanism, particularly in settings where long-term adaptation and moderately large label spaces make error accumulation a critical challenge. 
On the more challenging {ImageNetC} benchmark with 1,000 classes (Fig.~\ref{fig:inc}), VCoTTA delivers competitive results, obtaining 64.2\% compared with 66.7\% from CoTTA. 
A complementary stacked analysis of corruption groups further shows that VCoTTA distributes errors more evenly, suggesting that it accumulates errors more slowly under diverse and long-term shifts. 
These results confirm that priors inevitably drift during continual adaptation. 
By explicitly modeling uncertainty and adaptively combining priors, VCoTTA slows down this drift and improves robustness. 
The performance gains are especially pronounced on CIFAR10C and CIFAR100C, which serve as strong evidence of the method’s effectiveness.}



\subsection{Ablation Study on Auxiliary Components}
\lyu{To better understand the contribution of the auxiliary components, 
we conduct an ablation study while keeping the main framework of \textsc{VCoTTA} fixed. 
The backbone always includes variational inference with the uncertainty-aware prior mixture. 
We specifically examine the effect of the Variational Warm-Up (VWU) and the Symmetric Cross-Entropy (SCE). 
In Table~\ref{tab:ablation}, a checkmark indicates that the component is included, 
while absence means it is removed. 
In particular, ``w/o VWU'' means the BNN is trained directly without the warm-up stage, 
and ``w/o SCE'' means the symmetric cross-entropy is replaced with the standard cross-entropy loss.}
\lyu{The results show that both components provide additional improvements on top of the core framework. 
Removing VWU causes a clear drop in performance. 
For example, CIFAR10C decreases from 13.9\% to 18.4\%, and CIFAR100C decreases from 28.8\% to 31.5\%. 
This demonstrates that stochastic dynamics introduced during warm-up are important for stable adaptation. 
Adding SCE further improves accuracy on CIFAR10C and CIFAR100C, because it balances gradients between high-confidence and low-confidence predictions. 
On ImageNetC, the effect of SCE is marginal, which may be related to class imbalance and sensitivity in large-scale data. 
These findings confirm that VWU and SCE are effective auxiliary enhancements, 
while the main contribution of \textsc{VCoTTA} remains in the variational Bayesian formulation and the adaptive prior mixture. }

\begin{figure}[t]
\begin{center}
\centerline{\includegraphics[width=\linewidth]{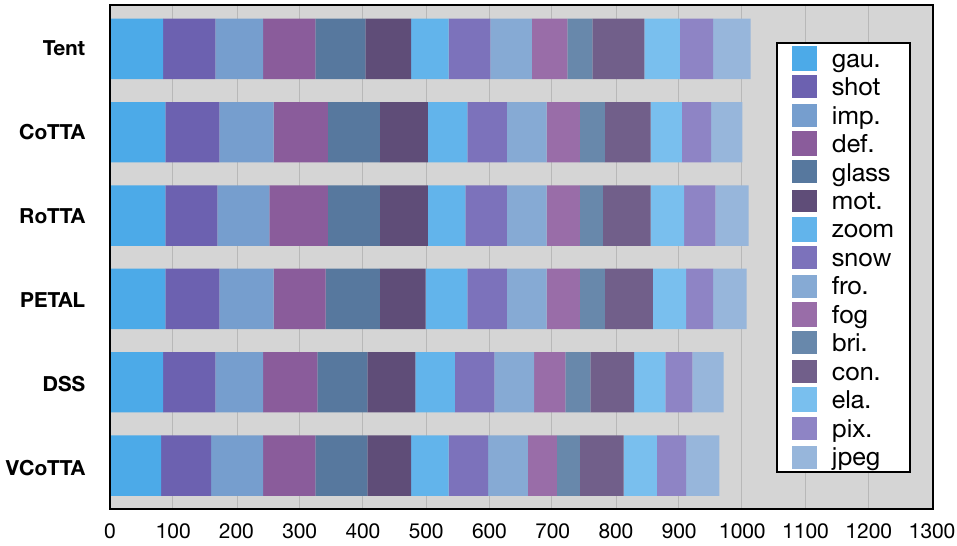}}
\caption{Sum of error rate (\%) on ImageNet-to-ImageNetC.}
\label{fig:inc}
\end{center}
\vspace{-15px}
\end{figure}

\begin{table}[t]
\caption{Ablation study of auxiliarxny components under corruption severity level 5. 
All variants keep the core framework of variational inference with the uncertainty-aware prior mixture unchanged. 
A checkmark indicates that the auxiliary component is used, while absence means it is removed.
}
\centering
\vspace{-5px}
\label{tab:ablation}
    \begin{tabular}{cccccc}
\toprule
VCoTTA & VWU & SCE  & CIFAR10C & CIFAR100C & ImageNetC\\
\midrule
$\surd$ &  &  & 18.4 &31.5& 68.1 \\
$\surd$ &  & $\surd$ & 17.1 &31.2& 68.3 \\
$\surd$ & $\surd$ &  & 13.9 &28.8& \textbf{64.2}\\
$\surd$ & $\surd$ & $\surd$ & \textbf{13.1} & \textbf{28.4} & 64.7  \\
\bottomrule
\end{tabular}
\vspace{-15px}
\end{table}

\begin{table}[h]
\vspace{-10px}
\centering
\label{tab:mix_prior}
\caption{Different weights for mixture of priors.
}
\begin{tabular}{cccccc}
\toprule
No. & $ \alpha$ & $1-\alpha$  & CIFAR10C & CIFAR100C & ImageNetC\\
\midrule
1 & 1 & 0 & 17.4 &35.0&69.9  \\
2 & 0.5 & 0.5 & 14.7 &31.3&67.0  \\
3 & 0 & 1 & 16.3 & 33.7 &71.2 \\
4 & \multicolumn{2}{c}{Eq.~\eqref{eq:alpha}} & \textbf{13.1} & \textbf{28.4} & \textbf{64.7} \\
\bottomrule
\end{tabular}
\vspace{-15px}
\end{table}

\subsection{Mixture of Priors}
In Sec.~\ref{sec:mixprior}, we introduce a Gaussian mixture strategy, where the current prior is approximated as the weighted sum of the source and teacher priors.
The weights are determined by computing the entropy over multiple augmentations of two models. 
To assess the effectiveness of these weights, we compare them with three naive weighting configurations: using only the source model, using only the teacher model, and a simple average with equal weights for both models.
The results, shown in Table~\ref{tab:mix_prior}, reveal that relying solely on the source model or the teacher model (i.e., weighting with $(1, 0)$ and $(0,1)$) results in suboptimal performance. 
Additionally, naive weighting with equal contributions from both models (i.e., $(0.5, 0.5)$) proves ineffective for CTTA due to the inherent uncertainty in both models.
In contrast, the proposed adaptive weights for the Gaussian mixture demonstrate its effectiveness. This underscores the significance of striking a balance between the two prior models in an unsupervised environment. The trade-off implies the need to discern when the source model's knowledge is more applicable and when the teacher model's shifting knowledge takes precedence.

\subsection{Uncertainty Estimation}

To evaluate the uncertainty estimation, we use negative loglikelihood (NLL) and Brier Score (BS)~\cite{brier1950verification}. Both NLL and BS are proper scoring rules~\cite{gneiting2007strictly}, and they are minimized if and only if the predicted distribution becomes identical to the actual distribution:
\begin{align}
    \text{NLL}= &
    -\mbb{E}_{(x,y)\in\mc{D}^\text{test}}\log(p(y|x, \boldsymbol{\theta})),\\
    \text{BS}= &
    \mbb{E}_{(x,y)\in\mc{D}^\text{test}}\left(p(y|x, \boldsymbol{\theta}) - \text{Onehot}(y)\right)^2,
\end{align}
where $\mc{D}^\text{test}$ denotes the test set, \ie, the unsupervised test dataset $\mc{U}$ with labels.
We evaluate NLL and BS with a severity level of 5 for all corruption types, and the compared results with SOTAs are shown in Table~\ref{tab:uncertainty}.
We have the following observations.
First, most methods suffer from low confidence in terms of NLL and BS because of the drift priors, where the model is unreliable gradually, and the error accumulation makes the model perform poorly.
Our approach outperforms most other approaches in terms of NLL and BS, demonstrating the superiority in improving uncertainty estimation. 
We also find that PETAL~\cite{brahma2023probabilistic} shows good NLL and BS, because PETAL forces the prediction over-confident to unreliable priors, thus PETAL shows unsatisfactory results on adaptation accuracy, such as 31.5\% vs. 28.4\% (Ours) on CIFAR100C.

\begin{table}[t]
\centering
\caption{Uncertainty estimation via NLL and BS.}
\label{tab:uncertainty}
\resizebox{\linewidth}{!}{
\begin{tabular}{lcccccc}
\toprule
\multirow{2}{*}{\textbf{Method}} & \multicolumn{2}{c}{\textbf{CIFAR10C}} & \multicolumn{2}{c}{\textbf{CIFAR100C}} & \multicolumn{2}{c}{\textbf{ImageNetC}}\\
 & NLL & BS & NLL & BS & NLL & BS \\
\midrule
Source & 3.0566 & 0.7478 & 2.4933 & 0.6707 &5.0703 & 0.9460\\
BN & 0.9988 & 0.3354 & 1.3932 & 0.4740 &3.9971 & 0.8345\\
Tent & 1.9391 & 0.3713 & 7.1097 &1.0838 & 3.6902& 0.8281 \\
CoTTA & 0.7192 & 0.2761 & 1.2907 & 0.4433 &3.6235 &\textbf{0.7972} \\
PETAL & 0.5899 & 0.2458 & \textbf{1.2267} & {0.4327} & 3.6391&{0.8017} \\
\midrule
VCoTTA & \textbf{0.5421} & \textbf{0.2130} & {1.2287} & \textbf{0.4307} & \textbf{3.4469} & 0.8092 \\
\bottomrule
\end{tabular}
}
\vspace{-15px}
\end{table}


\begin{table}[h]
\vspace{-10px}
\centering
\caption{Gradually changing on severity 5.}
\label{tab:gradual}
\begin{tabular}{lccc}
\toprule
\textbf{Method} &  \textbf{CIFAR10C} & \textbf{CIFAR100C} & \textbf{ImageNetC}\\
\midrule
Source  & 23.9  &32.9 & 81.7 \\
BN  & 13.5  & 29.7 & 54.1 \\
TENT  & 39.1 &72.7 &  53.7\\
CoTTA  & 10.6  & 26.3 & 42.1  \\
PETAL  & 10.5  & 27.1 & 60.5 \\
\midrule
VCoTTA  & \textbf{8.9} & \textbf{24.4}  & {\textbf{39.9}} \\
\bottomrule
\end{tabular}
\vspace{-15px}
\end{table}

\subsection{Gradually Corruption}

We also show gradual corruption results instead of constant severity in the major comparison, and the results are reported in Table~\ref{tab:gradual}.
Specifically, each corruption adopts the gradual changing sequence: $1\rightarrow 2\rightarrow 3\rightarrow 4\rightarrow 5\rightarrow 4\rightarrow 3\rightarrow 2\rightarrow 1$, where the severity level is the lowest 1 when corruption type changes, therefore, the type change is gradual.
The distribution shift within each type is also gradual.
Under this situation, our VCoTTA also outperforms other methods, such as 8.9\% vs. 10.5\% (\textsc{PETAL}) on CIFAR10C, and 24.4\% vs. 26.3\% (\textsc{CoTTA}) on CIFAR100C.
The results show that the proposed \textsc{VCoTTA} based on BNN is also effective when the distribution change is uncertain.

\subsection{Augmentation Analysis}

In our method, we use the standard augmentation following CoTTA~\cite{wang2022continual}.
In this subsection, we analyze the some characteristics via experiments.

\subsubsection{Confidence margin in augmentations}

\label{app:margin}
First, we analyze the margin $\epsilon$ in Eq.~\eqref{eq:teacher_aug}.
We experimentally validate different margins with more choices. 
Experimental results are shown in Tables~\ref{tab:margin}.
The results indicate that different datasets may require different margins to control confidence. 
Moreover, Eq.~\eqref{eq:teacher_aug} signifies that the reliable teacher likelihood is represented by the mean of its augmentations with $\epsilon$ more confidence than the teacher itself.
Tables~\ref{tab:margin} illustrates the selection of $\epsilon$ in our approach on CIFAR10C, CIFAR100C and ImageNetC. Note that when $\epsilon=-1$, it means no margin is used and the method will use all augmentated samples, i.e., without using Eq.~\eqref{eq:teacher_aug}. The results show that the proposed margin can effectively filter out unreliable augmented samples and achieve a better teacher log-likelihood.

\begin{table}[t]
\caption{Analysis on confidence margin.
}
\label{tab:margin}
\vspace{-10px}
\begin{center}
\begin{tabular}{c|cc|cc|cc}
\toprule
No. & $\epsilon$  & CIFAR10C & $\epsilon$ & CIFAR100C & $\epsilon$ & ImageNetC\\
\midrule
1  & -1 & 15.1 & -1 & 29.3 &-1 &66.4 \\
2  & 0 & 13.23 & 0 & 28.78 &0&65.0 \\
3  & 1e-4 & 13.23 & 0.1 & 28.55 &1e-3&65.0 \\
4  & 1e-3 & 13.22 & 0.2 & 28.45 &1e-2 &64.8 \\
5  & 1e-2 & \textbf{13.14} & 0.3 & \textbf{28.43} &1e-1&\textbf{64.7}\\
6  & 1e-1 & 13.31 & 0.4 & 28.54 &2e-1 &66.2 \\

\bottomrule
\end{tabular}
\end{center}
\vspace{-15px}
\end{table}

\subsubsection{Different number of augmentation}

In our method, we also use augmentation to enhance the confidence.
We then evaluate the the number of augmentation in Eq.~\eqref{eq:alpha}.
The results can be seen in Table~\ref{tab:no_aug}, and shows that increasing the number of augmentations can enhance effectiveness, but this hyperparameter ceases to have a significant impact after reaching 32.

\begin{table}[t]
\caption{Different number of augmentation.
}
\label{tab:no_aug}
\begin{center}
\vspace{-10px}
\begin{tabular}{lcccccc}
\toprule
Method & 0 & 4 & 8 & 16 & 32 & 64  \\
\midrule
CoTTA & 17.5 & 17.0 & 16.6 & 16.5 & 16.3 & 16.2   \\
PETAL & 17.3 &16.9  &16.4 &16.1 &16.0 &16.0 \\
\midrule
VCoTTA & \textbf{14.9} & \textbf{13.8} & \textbf{13.6} & \textbf{13.3} & \textbf{13.1} &\textbf{13.1} \\
\bottomrule
\end{tabular}
\end{center}
\vspace{-15px}
\end{table}

\subsection{Analysis of Variational Warm-up Strategy}

\label{app:vwu}

We have discussed the Variational Warm-Up (VWU) strategy in Section~\ref{sec:warmup}, and explain that the warm-up strategy is a common practice in TTA and CTTA.
We further illustrate some attributes of the proposed variational warm-up strategy.
In our method, the VWU strategy is used to turn an off-the-shelf model to a pretrained BNN.
The advantage of this approach is that pretrained models are readily available (e.g., directly leveraging official models in PyTorch), while pretrained BNNs are challenging to obtain, especially for large-scale datasets. Moreover, training BNNs is more difficult. Therefore, constructing BNN pretrained models based on existing pretrained models is a feasible approach. Additionally, we find that such a warm-up strategy requires only a few epochs to achieve satisfactory results. To validate the characteristics of the proposed VWU strategy, we designed the following experiments.


\subsubsection{Warm-up vs. Directly Pretraining BNN}

First, we conducted experiments to compare the performance of obtaining pretrained BNN models using the warm-up approach versus directly training the source model with BNN.
We pretrain the BNN also use VI as describing in Section~\ref{sec:warmup}.
The results can be seen in Fig.~\ref{fig:cnnvsbnn}.
As we can see, the results are at the same level, for example VI pretraining is with 13.2\% error rate while the proposed VWU achieves 13.1\% on CIFAR10C.
However, if we direct turn a pretrained model to a BNN by adding random stochastic parameters, without warm-up strategy, the results drop to 17.1\%.
This shows that VWU is a feasible strategy to obtain a pretrained BNN.


\begin{figure}[t]
\begin{center}
\centerline{\includegraphics[width=.9\linewidth]{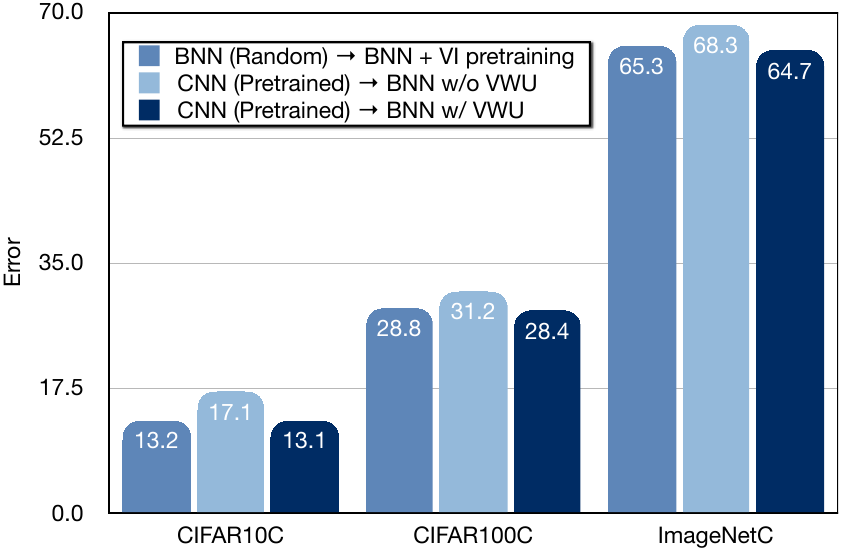}}
\vspace{-10px}
\caption{Comparison of variational warm-up and directly BNN pretraining.}
\label{fig:cnnvsbnn}
\end{center}
\vspace{-20px}
\end{figure}

\subsubsection{Number of Warm-up Epochs}

In our implementation, we employ only a limited number of epochs for variational warm-up, say 5 epochs. 
This is due to the fact that the pretrained model already fits well, thus requiring minimal adjustments to the mean of BNN. 
Additionally, the standard deviation (std) is initialized to be small.
Consequently, only a small number of iterations are necessary to update the BNN, and the step size is also kept small. Experimentation on the epoch number of variational warm-up reveals that keeping increasing epochs ( $>5$) will diminish performance, as shown in Fig.~\ref{fig:vwuepoch}.


\begin{figure}[t]
\centering
\subfigure[Different warm-up epochs.]{\includegraphics[width=.47\linewidth]{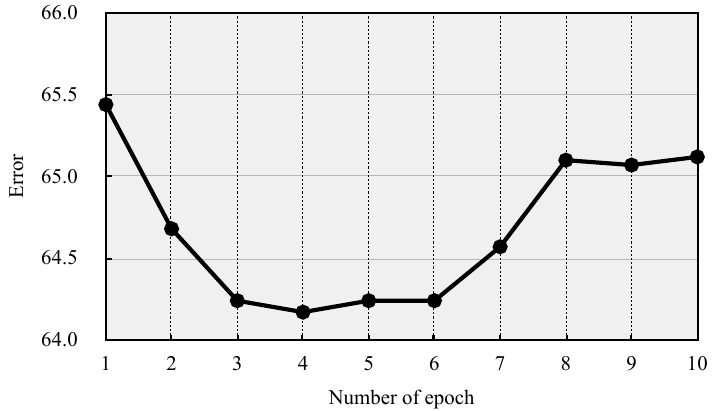} \label{fig:vwuepoch}
}
\subfigure[Different warm-up data scale.]{\includegraphics[width=.47\linewidth]{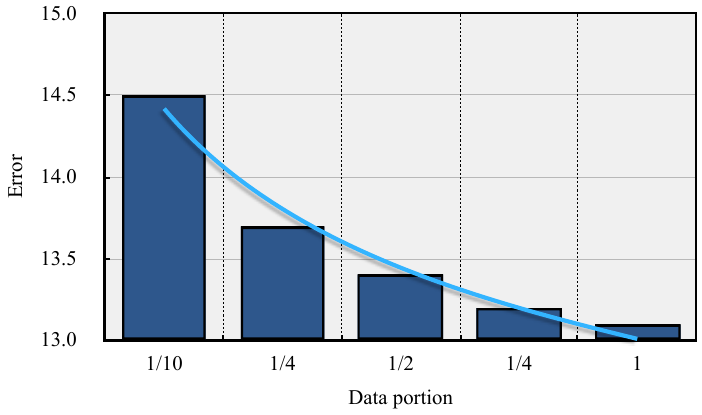} \label{fig:potion}
}
\caption{Comparisons on different warm-up settings of CIFAR10C.}
\vspace{-10px}
\end{figure}

\subsubsection{Only Portion Usage of Source Dataset in Warm-up}

The warm-up strategy is a common approach in TTA and CTTA tasks and it is regarded as a part of the pretraining stage. We also evaluate how we only use partial data for warm-up, and the results are shown in Fig.~\ref{fig:potion}. 
The experimental results demonstrate that a moderate reduction in sample size still maintains certain effectiveness of the warm-up strategy. However, excessive reduction, such as reducing to $1/10$, leads to a certain decline in effectiveness. This is because the warm-up strategy aims to incorporate statistical information from the dataset into the model, and insufficient data may result in inaccurate performance.

\subsection{Different Orders of Corruption}

\label{app:10order}

As we discuss in the major comparisons (see Sec~\ref{sec:ex.major}), the performance may be affected by the corruption order.
To provide a more comprehensive evaluation of the matter of the order, we conduct 10 different orders from Sec~\ref{sec:ex.major}, and show the average performance of all compared methods.
10 independent random orders of corruption are all under the severity level of 5. 
The results are shown in Table~\ref{tab:10orders}.
We find that the order of corruption is minor on simple datasets such as CIFAR10C and CIFAR100C, but small std on difficult datasets such as ImageNetC.
The proposed \textsc{VCoTTA} outperforms other methods on the average error of CIFAR10C and CIFAR100C under 10 different corruption orders, which shows the effectiveness of the prior calibration in CTTA.
Moreover, \textsc{VCoTTA} has comparable results with PETAL on ImageNetC, but smaller std over 10 orders, which shows the robustness of the proposed method.

\begin{table}[t]
\caption{Comparisons over 10 orders (avg $\pm$ std).
}
\label{tab:10orders}
\vspace{-10px}
\begin{center}
\begin{tabular}{cccc}
\toprule
Method  & CIFAR10C & CIFAR100C & ImageNetC\\
\midrule
CoTTA & 17.3$\pm$0.3 & 32.2$\pm$0.3  & 63.4$\pm$3.0 \\
PETAL & 16.0$\pm$0.1 & 33.8$\pm$0.3 & \textbf{62.7}$\pm$2.6 \\
\midrule
VCoTTA & \textbf{13.1$\pm$0.1} & \textbf{28.2$\pm$0.2} & 62.8$\pm$\textbf{1.1}\\
\bottomrule
\end{tabular}
\end{center}
\vspace{-15px}
\end{table}

\begin{table*}[t]
\caption{Resetting to the source model when domain shifts.}
\label{tab:reset}
\begin{center}
\vspace{-10px}
\begin{tabular}{lccccccccccccccccc}
\toprule
Method & Reset & C1 & C2 & C3 & C4 & C5 & C6 & C7 & C8 & C9 & C10 & C11 & C12 & C13 & C14 & C15 & avg \\
\midrule
CoTTA && 40.1 & 37.7 & 39.7 & 26.9 & 38.0 & 27.9 & 26.4 & 32.8 & 31.8 & 40.3 & 24.7 & 26.9 & 32.5 & 28.3 & 33.5 & 32.5 \\
CoTTA & $\surd$ & 40.1 & 40.2 & 42.5 & 27.4 & 41.6 & 29.5 & 27.6 & 34.6 & 34.8 & 41.2 & 26.4 & 30.0 & 35.3 & 32.6 & 40.8 & 35.1 \\
\midrule
VCoTTA && 35.3 & 32.8 & 38.9 & 23.8 & 34.6 & 25.5 & 23.2 & 27.5 & 26.7 & 30.4 & 22.1 & 23.0 & 28.1 & 24.2 & 30.4 & 28.4 \\
VCoTTA & $\surd$ & 35.3 & 33.7 & 38.7 & 24.0 & 37.0 & 26.2 & 24.1 & 29.2 & 28.6 & 33.9 & 22.2 & 23.7 & 30.9 & 26.2 & 34.5 & 29.9 \\
\bottomrule
\end{tabular}
\end{center}
\vspace{-15px}
\end{table*}

\subsection{Corruption Loops}

In the real-world scenario, the testing domain may reappear in the future.
We evaluate the test conditions continually 10 times to evaluate the long-term adaptation performance on CIFAR10C. 
That is, the test data will be re-inference and re-adapt for 9 more turns under severity 5.
Full results can be found in Fig.~\ref{fig:loop}. 
The results show that most compared methods obtain performance improvement in the first several loops, but suffer from performance drop in the following loops.
This means that the model drift can be even useful in early loops, but the drift becomes hard because of the unreliable prior.
The results also indicate that our method outperforms others in this long-term adaptation situation and has only small performance drops.


\begin{figure}[h]
\begin{center}
\vspace{-5px}
\centerline{\includegraphics[width=\linewidth]{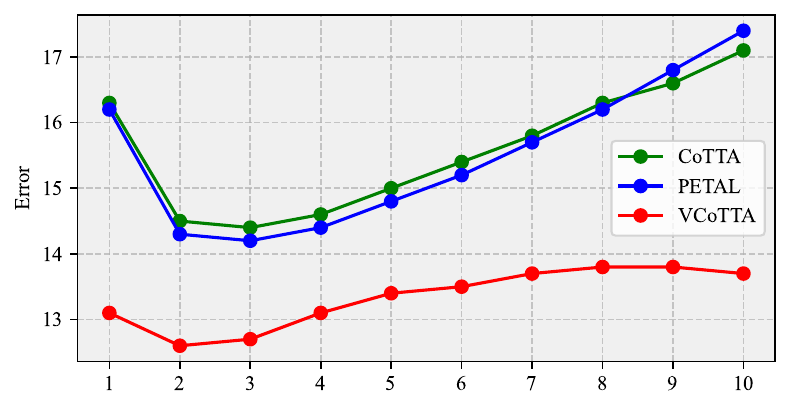}}
\caption{10 loops under a same corruption order (CIFAR10C).}
\label{fig:loop}
\end{center}
\vspace{-10px}
\end{figure}

\subsection{Experiment on different batch sizes}

CTTA does operate in an online setting, where all testing data is used only once. 
We evaluate our methods with different batch sizes.
For a testing dataset with a fixed number of samples, a larger batch size means fewer updates.
The current focus of CTTA research primarily revolves around batch-mode online settings, with batch sizes typically set to 200 in our experiments like other SOTAs. 
Moreover, strict online learning settings where each data point is processed individually are under-researched.
In fact, our method can be applied in scenarios with online learning or small batch sizes.
We experiment with the batch size of 1 on CIFAR10C, and compare the results with some baseline methods. The comparison results are shown in Fig.~\ref{fig:bs}.
Note that the batch normalization (BN) layers are disabled when the batch size is 1.
The results show that small batch sizes in CTTA result in worse performance.
We believe this is because a small batch size amplifies the uncertainty in model training.
We also find that TENT has high errors when the batch size is small but larger than 1. 
This is because TENT only updates the BN layers, which highly depends on the size of mini batch.


\begin{figure}[t]
\centering
\includegraphics[width=\linewidth]{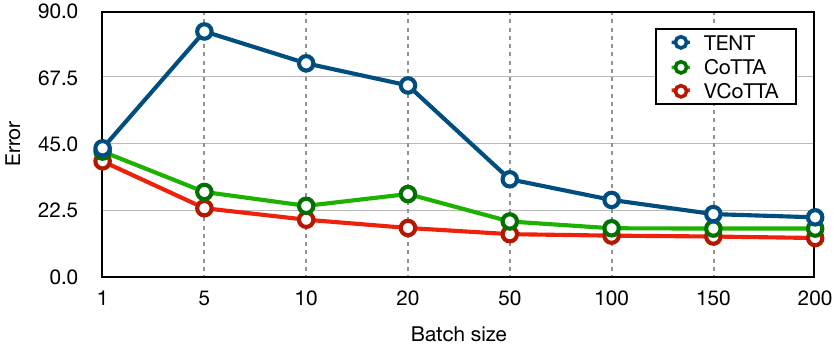}
\vspace{-10px}
\label{fig:bs}
\caption{
Error comparisons of different test batch sizes.
}
\end{figure}



\subsection{Can model reset mitigate the error accumulation issue?}

In the experiments of CTTA, it is interesting to know if resetting to source simply can solve the drift issue and achieve better results than CTTA methods.
We conduct this experiment on CoTTA and our method, the results can be seen in Table \ref{tab:reset}. We find that the resetting performance is slightly less effective than the continual setting on both CoTTA and our method. 
The results indicate that there may be shared knowledge between different tasks in the CTTA. 
Directly resetting to the source model would overlook this shared knowledge. Although the issue of error accumulation no longer needs to be considered, it becomes challenging to effectively mine some of the shared information.

\subsection{Time and Memory Cost}

\label{app:time}

We implement our method using a single RTX-4090 GPU card.
We provide the memory and time cost on CIFAR10C in Table~\ref{tab:ex_time}.
Our proposed VCoTTA method does not offer an advantage in terms of memory usage. This is because, in the BNN framework, additional standard deviations are required for implementing local reparameterization tricks. However, during the testing phase, this does not significantly impact the efficiency of the model. This is because, during testing, only the student model employs variational inference, which requires uncertainty parameters.

\begin{table}[t]
    \centering
    \caption{Time and memory cost comparisons.}
    \begin{tabular}{ccc}
    \toprule
        Method & Memory & Time per corruption \\
    \midrule
        CoTTA & 10.3Gb & 272s\\
        PETAL & 10.2Gb & 261s\\
        VCoTTA & 11.1Gb & 279s \\
     \bottomrule
    \end{tabular}
    \label{tab:ex_time}
\end{table}

\subsection{Different learning rate}

We also adjust the learning rate and select suitable learning rates for different datasets. As shown in the figure, through experimentation, we choose learning rates of $1e-4$, $1e-4$, and $1e-2$ for the CIFAR10C, CIFAR100C, and ImagenetC datasets, respectively. Note that this result is only applicable to the experimental scenario. 
It is challenging to select the optimal learning rate in a real-world testing environment, as the future testing scenarios are unknown.

\begin{figure}[t]
\centering
\includegraphics[width=.95\linewidth]{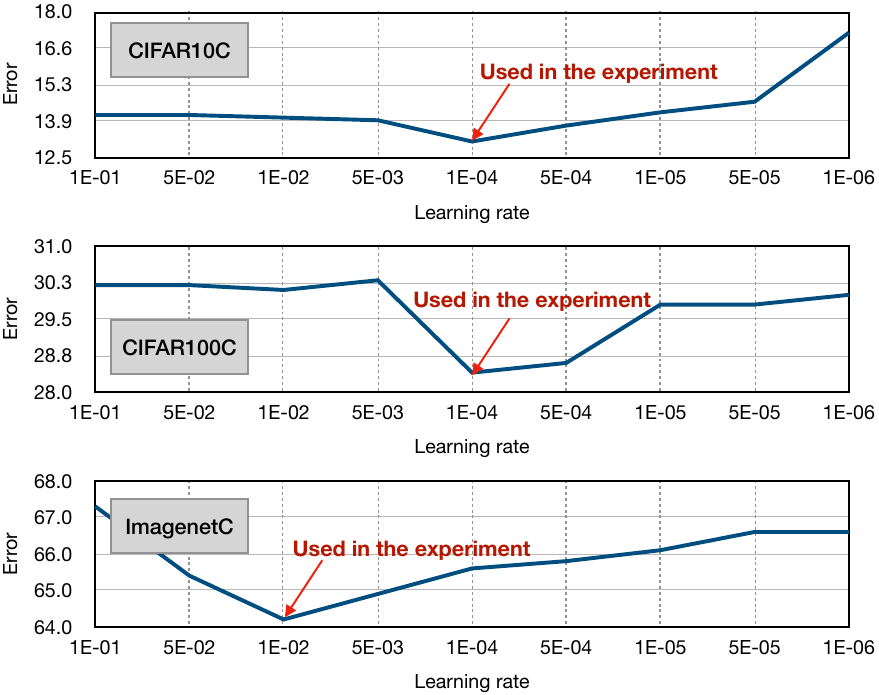}
\vspace{-10px}
\label{fig:lr}
\caption{
Learning rate analysis of the proposed method.
}
\end{figure}

\section{Conclusion and Limitation}

In this paper, we proposed a variational Bayesian inference approach, termed VCoTTA, to estimate uncertainties in CTTA. 
At the pretrained stage, we first transformed an off-the-shelf pretrained deterministic model into a BNN using a variational warm-up strategy, thereby injecting uncertainty into the source model.
At the test time, we implemented a mean-teacher update strategy, where the student model is updated via variational inference, while the teacher model is refined by the exponential moving average. 
Specifically, to update the student model, we proposed a novel approach that utilizes a mixture of priors from both the source and teacher models. 
Consequently, the ELBO can be formulated as the cross-entropy between the student and teacher models, combined with the KL divergence of the prior mixture.
We demonstrated the effectiveness of the proposed method on three datasets, and the results show that the proposed method can mitigate the issue of unreliable prior within the CTTA framework.

The efficacy of the proposed method relies on injecting uncertainty into the model during the pre-training phase, which may be unavailable in scenarios where pretraining is already completed, and original data is inaccessible. 
Additionally, constructing and training BNN models are inherently more complex, highlighting the importance of enhancing computational efficiency. The Gaussian mixture method relies on multiple data augmentations, which also incurs computational costs. Future endeavors could explore more efficient approaches for Gaussian mixtures.


%





\ifCLASSOPTIONcaptionsoff
  \newpage
\fi



\bibliographystyle{IEEEtran}
%


\bibliography{ref}

@inproceedings{nguyen2018variational,
  title={Variational continual learning},
  author={Nguyen, Cuong V and Li, Yingzhen and Bui, Thang D and Turner, Richard E},
  booktitle={Proceedings of the International Conference on Learning Representations},
  year={2018}
}

@inproceedings{grandvalet2004semi,
  title={Semi-supervised learning by entropy minimization},
  author={Grandvalet, Yves and Bengio, Yoshua},
  booktitle={Advances in Neural Information Processing Systems},
  year={2004}
}

@inproceedings{wang2022continual,
  title={Continual test-time domain adaptation},
  author={Wang, Qin and Fink, Olga and Van Gool, Luc and Dai, Dengxin},
  booktitle={Proceedings of the Computer Vision and Pattern Recognition},
  year={2022}
}

@book{box2011bayesian,
  title={Bayesian inference in statistical analysis},
  author={Box, George EP and Tiao, George C},
  year={2011},
  publisher={John Wiley \& Sons}
}

@inproceedings{ovadia2019can,
  title={Can you trust your model's uncertainty? evaluating predictive uncertainty under dataset shift},
  author={Ovadia, Yaniv and Fertig, Emily and Ren, Jie and Nado, Zachary and Sculley, David and Nowozin, Sebastian and Dillon, Joshua and Lakshminarayanan, Balaji and Snoek, Jasper},
  booktitle={Advances in Neural Information Processing Systems},
  year={2019}
}

@inproceedings{zhou2021training,
 author = {Zhou, Aurick and Levine, Sergey},
 booktitle = {Advances in Neural Information Processing Systems},
 title = {Bayesian adaptation for covariate shift},
 year = {2021}
}

@inproceedings{huang2022extrapolative,
  title={Extrapolative continuous-time bayesian neural network for fast training-free test-time adaptation},
  author={Huang, Hengguan and Gu, Xiangming and Wang, Hao and Xiao, Chang and Liu, Hongfu and Wang, Ye},
  booktitle={Advances in Neural Information Processing Systems},
  year={2022}
}

@article{zhao2022deep,
  title={Deep bayesian unsupervised lifelong learning},
  author={Zhao, Tingting and Wang, Zifeng and Masoomi, Aria and Dy, Jennifer},
  journal={Journal of the Neural Networks},
  volume={149},
  pages={95--106},
  year={2022},
}

@inproceedings{brahma2023probabilistic,
  title={A probabilistic framework for lifelong test-time adaptation},
  author={Brahma, Dhanajit and Rai, Piyush},
  booktitle={Proceedings of the Computer Vision and Pattern Recognition},
  year={2023}
}

@inproceedings{tarvainen2017mean,
  title={Mean teachers are better role models: Weight-averaged consistency targets improve semi-supervised deep learning results},
  author={Tarvainen, Antti and Valpola, Harri},
  booktitle={Advances in Neural Information Processing Systems},
  year={2017}
}

@article{chakrabarty2023sata,
  title={SATA: Source anchoring and target alignment network for continual test time adaptation},
  author={Chakrabarty, Goirik and Sreenivas, Manogna and Biswas, Soma},
  journal={arXiv preprint arXiv:2304.10113},
  year={2023}
}

@inproceedings{wang2020tent,
  title={Tent: Fully test-time adaptation by entropy minimization},
  author={Wang, Dequan and Shelhamer, Evan and Liu, Shaoteng and Olshausen, Bruno and Darrell, Trevor},
  booktitle={Proceedings of the International Conference on Learning Representations},
  year={2020}
}

@inproceedings{yang2023exploring,
  title={Exploring safety supervision for continual test-time domain adaptation},
  author={Yang, Xu and Gu, Yanan and Wei, Kun and Deng, Cheng},
  booktitle={Proceedings of the International Joint Conference on Artificial Intelligence},
  year={2023}
}

@inproceedings{bui2016deep,
  title={Deep Gaussian processes for regression using approximate expectation propagation},
  author={Bui, Thang and Hern{\'a}ndez-Lobato, Daniel and Hernandez-Lobato, Jose and Li, Yingzhen and Turner, Richard},
  booktitle={Proceedings of the International Conference on Machine Learning},
  year={2016},
}

@article{liu2019variational,
  title={Variational inference with Gaussian mixture model and householder flow},
  author={Liu, GuoJun and Liu, Yang and Guo, MaoZu and Li, Peng and Li, MingYu},
  journal={Journal of the Neural Networks},
  volume={109},
  pages={43--55},
  year={2019},
}

@inproceedings{kingma2015variational,
  title={Variational dropout and the local reparameterization trick},
  author={Kingma, Durk P and Salimans, Tim and Welling, Max},
  booktitle={Advances in Neural Information Processing Systems},
  year={2015}
}

@inproceedings{wang2011online,
  title={Online variational inference for the hierarchical Dirichlet process},
  author={Wang, Chong and Paisley, John and Blei, David M},
  booktitle={Proceedings of the International Conference on Artificial Intelligence and Statistics},
  year={2011},
}

@article{sato2001online,
  title={Online model selection based on the variational Bayes},
  author={Sato, Masa-Aki},
  journal={Journal of the Neural Computation},
  volume={13},
  pages={1649--1681},
  year={2001},
}

@inproceedings{singer1998batch,
  title={Batch and on-line parameter estimation of Gaussian mixtures based on the joint entropy},
  author={Singer, Yoram and Warmuth, Manfred KK},
  booktitle={Procedings of the Advances in Neural Information Processing Systems},
  year={1998}
}

@book{cover1999elements,
  title={Elements of information theory},
  author={Cover, Thomas M},
  year={1999},
  publisher={John Wiley \& Sons}
}

@inproceedings{jain2011online,
  title={Online domain adaptation of a pre-trained cascade of classifiers},
  author={Jain, Vidit and Learned-Miller, Erik},
  booktitle={Proceedings of the Computer Vision and Pattern Recognition},
  year={2011},
}

@inproceedings{sun2020test,
  title={Test-time training with self-supervision for generalization under distribution shifts},
  author={Sun, Yu and Wang, Xiaolong and Liu, Zhuang and Miller, John and Efros, Alexei and Hardt, Moritz},
  booktitle={Procedings of the International Conference on Machine Learning},
  year={2020},
}

@inproceedings{hernandez2015probabilistic,
  title={Probabilistic backpropagation for scalable learning of bayesian neural networks},
  author={Hern{\'a}ndez-Lobato, Jos{\'e} Miguel and Adams, Ryan},
  booktitle={Procedings of the International Conference on Machine Learning},
  year={2015},
}

@inproceedings{louizos2017multiplicative,
  title={Multiplicative normalizing flows for variational bayesian neural networks},
  author={Louizos, Christos and Welling, Max},
  booktitle={Procedings of the International Conference on Machine Learning},
  year={2017},
}

@inproceedings{gal2016dropout,
  title={Dropout as a bayesian approximation: Representing model uncertainty in deep learning},
  author={Gal, Yarin and Ghahramani, Zoubin},
  booktitle={Procedings of the International Conference on Machine Learning},
  year={2016},
}

@inproceedings{ritter2018scalable,
  title={A scalable laplace approximation for neural networks},
  author={Ritter, Hippolyt and Botev, Aleksandar and Barber, David},
  booktitle={International Conference on Learning Representations},
  year={2018},
}

@article{friston2007variational,
  title={Variational free energy and the Laplace approximation},
  author={Friston, Karl and Mattout, J{\'e}r{\'e}mie and Trujillo-Barreto, Nelson and Ashburner, John and Penny, Will},
  journal={Journal of the Neuroimage},
  volume={34},
  number={1},
  pages={220--234},
  year={2007},
}

@book{gelman1995bayesian,
  title={Bayesian data analysis},
  author={Gelman, Andrew and Carlin, John B and Stern, Hal S and Rubin, Donald B},
  year={1995},
  publisher={Chapman and Hall/CRC}
}

@inproceedings{fortuin2021bayesian,
  title={Bayesian Neural Network Priors Revisited},
  author={Fortuin, Vincent and Garriga-Alonso, Adri{\`a} and Ober, Sebastian W and Wenzel, Florian and Ratsch, Gunnar and Turner, Richard E and van der Wilk, Mark and Aitchison, Laurence},
  booktitle={International Conference on Learning Representations},
  year={2021}
}

@inproceedings{ebrahimi2019uncertainty,
  title={Uncertainty-guided Continual Learning with Bayesian Neural Networks},
  author={Ebrahimi, Sayna and Elhoseiny, Mohamed and Darrell, Trevor and Rohrbach, Marcus},
  booktitle={International Conference on Learning Representations},
  year={2019}
}

@article{farquhar2019unifying,
  title={A unifying bayesian view of continual learning},
  author={Farquhar, Sebastian and Gal, Yarin},
  journal={arXiv preprint arXiv:1902.06494},
  year={2019}
}

@inproceedings{kurle2019continual,
  title={Continual learning with bayesian neural networks for non-stationary data},
  author={Kurle, Richard and Cseke, Botond and Klushyn, Alexej and Van Der Smagt, Patrick and G{\"u}nnemann, Stephan},
  booktitle={International Conference on Learning Representations},
  year={2019}
}

@inproceedings{zagoruyko2016wide,
  title={Wide Residual Networks},
  author={Zagoruyko, Sergey and Komodakis, Nikos},
  booktitle={Procedings of the British Machine Vision Conference},
  year={2016},
}

@inproceedings{xie2017aggregated,
  title={Aggregated residual transformations for deep neural networks},
  author={Xie, Saining and Girshick, Ross and Doll{\'a}r, Piotr and Tu, Zhuowen and He, Kaiming},
  booktitle={Proceedings of the Computer Vision and Pattern Recognition},
  year={2017}
}

@inproceedings{he2016deep,
  title={Deep residual learning for image recognition},
  author={He, Kaiming and Zhang, Xiangyu and Ren, Shaoqing and Sun, Jian},
  booktitle={Proceedings of the Computer Vision and Pattern Recognition},
  year={2016}
}

@article{brier1950verification,
  title={Verification of forecasts expressed in terms of probability},
  author={Brier, Glenn W},
  journal={Journal of the Monthly Weather Review},
  volume={78},
  number={1},
  pages={1--3},
  year={1950}
}

@article{gneiting2007strictly,
  title={Strictly proper scoring rules, prediction, and estimation},
  author={Gneiting, Tilmann and Raftery, Adrian E},
  journal={Journal of the American Statistical Association},
  volume={102},
  number={477},
  pages={359--378},
  year={2007},
}

@inproceedings{wang2019symmetric,
  title={Symmetric cross entropy for robust learning with noisy labels},
  author={Wang, Yisen and Ma, Xingjun and Chen, Zaiyi and Luo, Yuan and Yi, Jinfeng and Bailey, James},
  booktitle={Proceedings of the Computer Vision and Pattern Recognition},
  year={2019}
}

@article{hunter1986exponentially,
  title={The exponentially weighted moving average},
  author={Hunter, J Stuart},
  journal={Journal of the Quality Technology},
  volume={18},
  number={4},
  pages={203--210},
  year={1986},
}

@article{van2020survey,
  title={A survey on semi-supervised learning},
  author={Van Engelen, Jesper E and Hoos, Holger H},
  journal={Machine Learning},
  volume={109},
  number={2},
  pages={373--440},
  year={2020},
  publisher={Springer}
}

@inproceedings{caron2018deep,
  title={Deep clustering for unsupervised learning of visual features},
  author={Caron, Mathilde and Bojanowski, Piotr and Joulin, Armand and Douze, Matthijs},
  booktitle={Proceedings of the European Conference on Computer Vision},
  pages={132--149},
  year={2018}
}

@article{chen2023each,
  title={Each Test Image Deserves A Specific Prompt: Continual Test-Time Adaptation for 2D Medical Image Segmentation},
  author={Chen, Ziyang and Ye, Yiwen and Lu, Mengkang and Pan, Yongsheng and Xia, Yong},
  journal={arXiv preprint arXiv:2311.18363},
  year={2023}
}

@inproceedings{song2023ecotta,
  title={Ecotta: Memory-efficient continual test-time adaptation via self-distilled regularization},
  author={Song, Junha and Lee, Jungsoo and Kweon, In So and Choi, Sungha},
  booktitle={Proceedings of the IEEE/CVF Conference on Computer Vision and Pattern Recognition},
  pages={11920--11929},
  year={2023}
}

@inproceedings{niu2022efficient,
  title={Efficient test-time model adaptation without forgetting},
  author={Niu, Shuaicheng and Wu, Jiaxiang and Zhang, Yifan and Chen, Yaofo and Zheng, Shijian and Zhao, Peilin and Tan, Mingkui},
  booktitle={Proceedings of the International Conference on Machine Learning},
  pages={16888--16905},
  year={2022},
}

@inproceedings{jung2023cafa,
  title={Cafa: Class-aware feature alignment for test-time adaptation},
  author={Jung, Sanghun and Lee, Jungsoo and Kim, Nanhee and Shaban, Amirreza and Boots, Byron and Choo, Jaegul},
  booktitle={Proceedings of the IEEE/CVF International Conference on Computer Vision},
  pages={19060--19071},
  year={2023}
}

@inproceedings{dobler2023robust,
  title={Robust mean teacher for continual and gradual test-time adaptation},
  author={D{\"o}bler, Mario and Marsden, Robert A and Yang, Bin},
  booktitle={Proceedings of the IEEE/CVF Conference on Computer Vision and Pattern Recognition},
  pages={7704--7714},
  year={2023}
}

@inproceedings{wang2024continual,
  title={Continual test-time domain adaptation via dynamic sample selection},
  author={Wang, Yanshuo and Hong, Jie and Cheraghian, Ali and Rahman, Shafin and Ahmedt-Aristizabal, David and Petersson, Lars and Harandi, Mehrtash},
  booktitle={Proceedings of the IEEE/CVF Winter Conference on Applications of Computer Vision},
  pages={1701--1710},
  year={2024}
}

@inproceedings{yuan2023robust,
  title={Robust test-time adaptation in dynamic scenarios},
  author={Yuan, Longhui and Xie, Binhui and Li, Shuang},
  booktitle={Proceedings of the IEEE/CVF Conference on Computer Vision and Pattern Recognition},
  pages={15922--15932},
  year={2023}
}

@inproceedings{maddox2019simple,
  title={A simple baseline for bayesian uncertainty in deep learning},
  author={Maddox, Wesley J and Izmailov, Pavel and Garipov, Timur and Vetrov, Dmitry P and Wilson, Andrew Gordon},
  booktitle={Advances in Neural Information Processing Systems},
  year={2019}
}

@inproceedings{miller1996mixture,
  title={A mixture of experts classifier with learning based on both labelled and unlabelled data},
  author={Miller, David J and Uyar, Hasan},
  booktitle={Advances in Neural Information Processing Systems},
  year={1996}
}

@inproceedings{gan2023decorate,
  title={Decorate the newcomers: Visual domain prompt for continual test time adaptation},
  author={Gan, Yulu and Bai, Yan and Lou, Yihang and Ma, Xianzheng and Zhang, Renrui and Shi, Nian and Luo, Lin},
  booktitle={Proceedings of the AAAI Conference on Artificial Intelligence},
  volume={37},
  number={6},
  pages={7595--7603},
  year={2023}
}

@article{liu2023vida,
  title={Vida: Homeostatic visual domain adapter for continual test time adaptation},
  author={Liu, Jiaming and Yang, Senqiao and Jia, Peidong and Zhang, Renrui and Lu, Ming and Guo, Yandong and Xue, Wei and Zhang, Shanghang},
  journal={arXiv preprint arXiv:2306.04344},
  year={2023}
}

@inproceedings{liu2024continual,
  title={Continual-MAE: Adaptive Distribution Masked Autoencoders for Continual Test-Time Adaptation},
  author={Liu, Jiaming and Xu, Ran and Yang, Senqiao and Zhang, Renrui and Zhang, Qizhe and Chen, Zehui and Guo, Yandong and Zhang, Shanghang},
  booktitle={Proceedings of the IEEE/CVF Conference on Computer Vision and Pattern Recognition},
  pages={28653--28663},
  year={2024}
}

@inproceedings{becotta,
  title={BECoTTA: Input-dependent Online Blending of Experts for Continual Test-time Adaptation},
  author={Lee, Daeun and Yoon, Jaehong and Hwang, Sung Ju},
  journal={International Conference on Machine Learning},
  year={2024}
}

@inproceedings{EATA,
  title={Efficient test-time model adaptation without forgetting},
  author={Niu, Shuaicheng and Wu, Jiaxiang and Zhang, Yifan and Chen, Yaofo and Zheng, Shijian and Zhao, Peilin and Tan, Mingkui},
  booktitle={International Conference on Machine Learning},
  year={2022},
}

@inproceedings{AdaContrast,
  title={Contrastive test-time adaptation},
  author={Chen, Dian and Wang, Dequan and Darrell, Trevor and Ebrahimi, Sayna},
  booktitle={Proceedings of the IEEE/CVF Conference
on Computer Vision and Pattern Recognition},
  year={2022}
}

@inproceedings{obao,
  title={Reshaping the Online Data Buffering and Organizing Mechanism for Continual Test-Time Adaptation},
  author={Zhu, Zhilin and Hong, Xiaopeng and Ma, Zhiheng and Zhuang, Weijun and Ma, Yaohui and Dai, Yong and Wang, Yaowei},
  booktitle={European Conference on Computer Vision},
  pages={415--433},
  year={2024}
}

@inproceedings{palm,
  title={Palm: Pushing adaptive learning rate mechanisms for continual test-time adaptation},
  author={Maharana, Sarthak Kumar and Zhang, Baoming and Guo, Yunhui},
  booktitle={Proceedings of the AAAI Conference on Artificial Intelligence},
  volume={39},
  number={18},
  pages={19378--19386},
  year={2025}
}

@inproceedings{surgeon,
  title={SURGEON: Memory-Adaptive Fully Test-Time Adaptation via Dynamic Activation Sparsity},
  author={Ke Ma and Jiaqi Tang and Bin Guo and Fan Dang and Sicong Liu and Zhui Zhu and Lei Wu and Cheng Fang and Ying-Cong Chen and Zhiwen Yu and Yunhao Liu},
  booktitle={CVPR},
  year={2025}
}

@inproceedings{CTTA-TCA,
  title={Maintaining consistent inter-class topology in continual test-time adaptation},
  author={Ni, Chenggong and Lyu, Fan and Tan, Jiayao and Hu, Fuyuan and Yao, Rui and Zhou, Tao},
  booktitle={Proceedings of the Computer Vision and Pattern Recognition Conference},
  pages={15319--15328},
  year={2025}
}

@article{liu2023test,
  title={Test-time adaptation for nighttime color-thermal semantic segmentation},
  author={Liu, Yexin and Zhang, Weiming and Zhao, Guoyang and Zhu, Jinjing and Vasilakos, Athanasios V and Wang, Lin},
  journal={IEEE Transactions on Artificial Intelligence},
  volume={5},
  number={10},
  pages={4893--4904},
  year={2023},
  publisher={IEEE}
}

@article{verma2025real,
  title={Real-World Video Quality Assessment via Test-Time Adaptation and its application in Real-World Video Super-Resolution},
  author={Verma, Ajeet Kumar and Mishra, Ambuj and Jakhetiya, Vinit and Subudhi, Badri Narayan and Jaiswal, Sunil},
  journal={IEEE Transactions on Artificial Intelligence},
  year={2025},
  publisher={IEEE}
}

@article{kim2021domain,
  title={Domain adaptation without source data},
  author={Kim, Youngeun and Cho, Donghyeon and Han, Kyeongtak and Panda, Priyadarshini and Hong, Sungeun},
  journal={IEEE Transactions on Artificial Intelligence},
  volume={2},
  number={6},
  pages={508--518},
  year={2021},
  publisher={IEEE}
}

@article{wickramasinghe2023continual,
  title={Continual learning: A review of techniques, challenges, and future directions},
  author={Wickramasinghe, Buddhi and Saha, Gobinda and Roy, Kaushik},
  journal={IEEE Transactions on Artificial Intelligence},
  volume={5},
  number={6},
  pages={2526--2546},
  year={2023},
  publisher={IEEE}
}

@article{dai2025prompt,
  title={Prompt Customization for Continual Learning},
  author={Dai, Yong and Hong, Xiaopeng and Wang, Yabin and Ma, Zhiheng and Jiang, Dongmei and Wang, Yaowei},
  journal={IEEE Transactions on Artificial Intelligence},
  year={2025},
  publisher={IEEE}
}

\end{document}